\documentclass{article}
\usepackage{arxiv}
\usepackage[authoryear,round]{natbib}
\setcitestyle{authoryear,round,citesep={;},aysep={,},yysep={;}}
\renewcommand{\shorttitle}{SCOPE: Sparse PDE Inference}
\renewcommand{\headeright}{Preprint}
\date{}

\usepackage{amsmath,amssymb,amsthm,mathtools,booktabs,array}
\usepackage{multirow}
\usepackage{hyperref}
\usepackage{url,xurl,longtable} 
\usepackage{graphicx}
\usepackage{pifont} 
\usepackage{xcolor}
\usepackage{colortbl}
\definecolor{ScopeInk}{HTML}{214A55}
\definecolor{ScopeTint}{HTML}{EAF2F3}
\definecolor{ScopeGray}{HTML}{F3F4F5}
\newcommand{\scopepanel}[2]{\rowcolor{ScopeGray}\multicolumn{#1}{l}{\textcolor{ScopeInk}{\textbf{#2}}}\\}
\newcommand{\tabnote}[1]{\par\vspace{3pt}\begin{minipage}{\linewidth}\footnotesize #1\end{minipage}}
\newenvironment{reviewtext}{\begingroup\color{red}}{\endgroup}

\newcommand{\E}{\mathbb E}
\newcommand{\R}{\mathbb R}
\newcommand{\G}{\mathcal G}
\newcommand{\C}{\mathcal C}

\newcommand{\J}{\mathcal J}

\newcommand{\RX}[1]{}
\newcolumntype{L}[1]{>{\raggedright\arraybackslash}p{#1}}

\newtheorem{theorem}{Theorem}
\newtheorem{proposition}[theorem]{Proposition}

\theoremstyle{definition}
\newtheorem{definition}{Definition}

\theoremstyle{remark}

\newcommand{\figuretbd}[1]{%
    \begin{center}
        \fbox{%
            \parbox[c][1.25cm][c]{.93\linewidth}{%
                \centering\large TBD
            }%
        }
    \end{center}
}

\title{SCOPE: Observation-Conditioned Full-Target Prediction for Sparse PDE Inference}

\hypersetup{
    colorlinks=false,
    pdfborder={0 0 1},
    pdfborderstyle={},
    pdftitle={SCOPE: Observation-Conditioned Full-Target Prediction for Sparse PDE Inference},
    pdfauthor={}
}

\newcommand{\scopepublicauthors}{%
\begin{minipage}{\dimexpr\textwidth-2\tabcolsep\relax}
\centering\normalfont
{\large\bfseries
Ruichen Xu\textsuperscript{1,\ensuremath{\dagger}},\quad Siyao Wang\textsuperscript{2},\quad Fang Wan\textsuperscript{1},\quad Jiacheng Qiu\textsuperscript{1},\\[4pt]
Wenhan Gao\textsuperscript{1},\quad Jiaxing Zhang\textsuperscript{3},\quad Linsey Pang\textsuperscript{4,6},\quad Ravid Shwartz-Ziv\textsuperscript{5},\\[4pt]
Prakhar Mehrotra\textsuperscript{4},\quad Yann LeCun\textsuperscript{5},\quad Yuefan Deng\textsuperscript{1,\ensuremath{\dagger}}\par}
\vspace{8pt}
{\small
\textsuperscript{1}Stony Brook University\quad \textsuperscript{2}University of California, Davis\\[3pt]
\textsuperscript{3}Independent Research\quad \textsuperscript{4}PayPal\\[3pt]
\textsuperscript{5}New York University\quad \textsuperscript{6}Northeastern University\par}
\vspace{6pt}
{\small Code: \url{https://github.com/ru1ch3n/SCOPE}}
\end{minipage}}
\author{\scopepublicauthors}
\makeatletter
\g@addto@macro\@thanks{%
  \footnotetext[2]{Corresponding authors: Ruichen Xu
  (\href{mailto:ruichen.xu@stonybrook.edu}{\nolinkurl{ruichen.xu@stonybrook.edu}})
  and Yuefan Deng
  (\href{mailto:yuefan.deng@stonybrook.edu}{\nolinkurl{yuefan.deng@stonybrook.edu}}).}%
}
\makeatother
\hypersetup{pdfauthor={Ruichen Xu; Siyao Wang; Fang Wan; Jiacheng Qiu; Wenhan Gao; Jiaxing Zhang; Linsey Pang; Ravid Shwartz-Ziv; Prakhar Mehrotra; Yann LeCun; Yuefan Deng}}

\begin{document}
\maketitle

\begin{abstract}
Recovering complete physical fields from sparse observations is challenging because the measurements may not uniquely determine the underlying state. Diffusion-based PDE solvers address this problem through iterative sampling whereas neural operators provide deterministic one-pass predictions. We propose \textbf{SCOPE} (\textbf{\underline{S}parse-\underline{C}ontext \underline{O}bservability-aware \underline{P}redictive \underline{E}mbeddings}) to recover complete PDE fields from sparse observations by coupling full-field latent prediction with physical reconstruction. A shared decoder reconstructs fields from both predicted and complete-view representations so that representation learning is guided by both physical recovery and latent matching. We derive a quadratic risk decomposition at fixed teacher--decoder pairs showing why optimal latent prediction need not yield optimal field reconstruction. We also establish sufficient conditions for decoder improvements on complete inputs to transfer to recovery from partial observations. Experiments across five PDE settings show that SCOPE outperforms mask-aware neural operators on all ten forward and inverse tasks and achieves lower errors than those reported for diffusion-based solvers including DiffusionPDE and FunDPS. Decoder-only adaptation further improves recovery without retraining the backbone while retaining deterministic single-pass inference.
\end{abstract}
\section{Introduction}
\label{sec:introduction}

\begin{figure}[t]
\centering
\includegraphics[width=\linewidth]{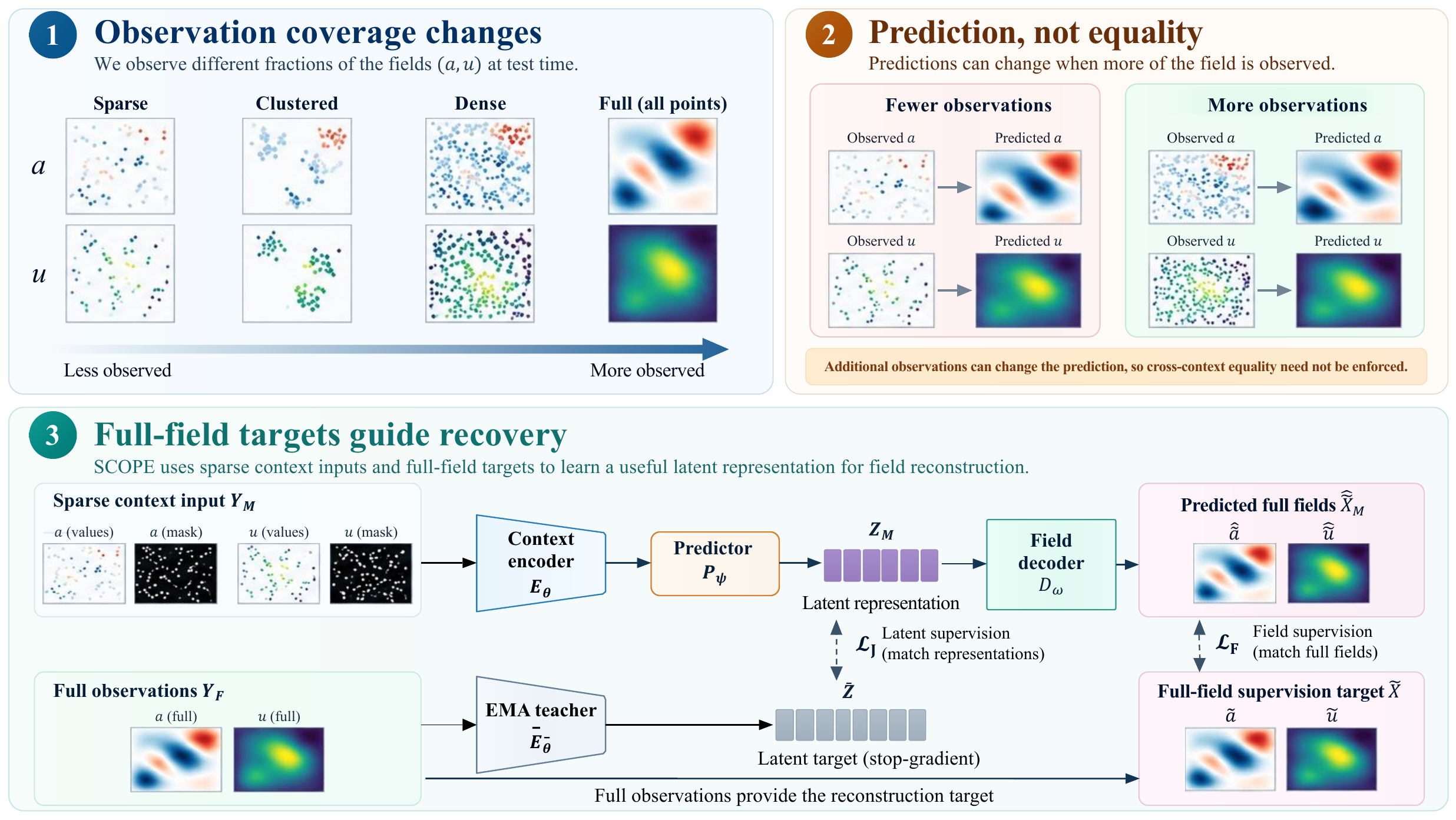}
\caption{\textbf{Overview of SCOPE.} \textbf{(1) Observation coverage varies across contexts:} different densities and layouts provide alternative observations of the same physical realization. \textbf{(2) Prediction, not equality:} predictions may change with additional evidence despite a common complete-field target. \textbf{(3) Full-field targets guide recovery:} complete training pairs separately provide teacher representations for latent prediction $\mathcal L_J$ and physical fields for reconstruction $\mathcal L_F$. Joint recovery is illustrated. Figure~\ref{fig:architecture} details the training pathways; Section~\ref{sec:scope_objective} describes decoder adaptation. Field maps are schematic, not experimental results.}
\label{fig:scope_overview}
\end{figure}

Partial differential equations (PDEs) describe physical systems across science and engineering~\citep{evans2010partial,quarteroni2009numerical,karniadakis2021physics}. Forward problems predict responses from physical inputs, while inverse problems recover unknown inputs from observed responses~\citep{kaipio2005statistical,stuart2010inverse}. Neural operators learn surrogate mappings between function spaces, producing field predictions through a network evaluation after training~\citep{li2021fno,lu2021deeponet,kovachki2023neural}. With incomplete measurements, however, predictors must use partial observations of an input field, a response field, or both~\citep{asch2016data,reichle2008data}. This creates two distinct challenges: \ding{182}~\textbf{Incomplete information}, because the evidence may not uniquely determine the fields; and \ding{183}~\textbf{Changing observation coverage}, because measurement densities and layouts vary. We study paired-field recovery with one model per PDE setting and complete paired fields for training supervision.

Diffusion-based methods condition field generation on partial observations~\citep{huang2024diffusionpde,yao2025fundps}. Predictive representation learning offers a complementary context--target approach~\citep{assran2023ijepa}, including prediction of contextualized full-input representations from partial inputs~\citep{baevski2022data2vec}. For physical recovery, the question is how to connect these representation targets to the observations and reconstructed fields. \textbf{Spatial coverage does not generally identify latent recoverability}: measurements may cover only part of a token's footprint, while spatially mixed features may depend on values elsewhere. We therefore use physical masks to condition prediction, not as recoverability labels for selecting latent targets. This motivates \textbf{a common full-target objective across observation contexts}, coupled to physical reconstruction, without assuming that every latent feature is exactly recoverable from every context.

We propose \textbf{SCOPE}: \underline{S}parse-\underline{C}ontext \underline{O}bservability-aware \underline{P}redictive \underline{E}mbeddings. Here, \emph{observability-aware} denotes conditioning on measurements and their availability, not estimating an observable latent subspace. A mask-conditioned student predicts the complete representation of an exponential-moving-average teacher with uniform relative coordinate weights. A shared nonlinear decoder reconstructs the paired fields from both the student's prediction and the online encoder's full-view representation, training both pathways for physical reconstruction. To refine the readout without retraining the representation, we freeze the encoder, predictor, and conditioner and fit a replacement decoder to their sparse-context representations. This \emph{sparse-path decoder-only adaptation} separates decoder optimization from backbone training and uses complete training targets offline, without optimizing on test examples. Both native and adapted models retain deterministic, single-pass inference from the supplied context.

Our contributions are threefold. \textbf{(1)} We develop a field-grounded predictive framework combining observation-conditioned full-target regression, shared physical reconstruction, and sparse-path decoder adaptation. \textbf{(2)} We formalize latent recoverability and analyze recovery at fixed learned representations. A quadratic reference separates observation uncertainty, decoder incompatibility, and predictor excess risk; a separate physical-error bound gives sufficient conditions for complete-to-sparse decoder transfer (Section~\ref{sec:lossaware} and Appendices~\ref{app:scope_recoverability}--\ref{app:scope_nonlinear}). \textbf{(3)} We evaluate sparse forward and inverse recovery across elliptic and flow problems, separating backbone-recipe comparisons from adaptation at a fixed Full backbone. The experiments demonstrate competitive results across the evaluated tasks; training settings, observation protocols, and decoder-adaptation costs are detailed in the appendix.

\section{Problem Setting and Related Work}
\label{sec:setting}

\subsection{PDE Inference under Partial Observations}
\label{sec:partial_observation_setting}

\emph{PDE-based inverse problems} recover unknown coefficients, sources, or initial states from response measurements~\citep{kaipio2005statistical,stuart2010inverse}. Classically, $y=\mathcal O(\mathcal F(a))+\varepsilon$, where $\mathcal F:\mathcal A\rightarrow\mathcal U$ is the PDE solution operator, $\mathcal O$ the measurement process, and $\varepsilon$ possible measurement noise. Partial measurements may admit multiple physical states, making recovery depend on both the observations and the field distribution~\citep{stuart2010inverse,asch2016data}. We consider conditional paired-field recovery from partial observations of either channel or both.

Let $X=(a,u)\in\mathbb R^{2N}$ denote paired fields on a fixed grid $\Omega_h$, with \textbf{one independently trained model per PDE setting} handling its different observation tasks. The channels represent source/solution, permeability/pressure, or scalar flow snapshots, as specified with the grid in Appendix~\ref{app:data}. Flow channels are not velocity components; the setting concerns paired-field recovery rather than multi-PDE pretraining or autoregressive rollout.

Write $\widetilde X=\mathcal N(X)=(\widetilde a,\widetilde u)$ for the pair under the fixed released normalization. For masks $M=(M_a,M_u)$, the input is $Y_M=(M_a\odot\widetilde a,M_a,M_u\odot\widetilde u,M_u)$, where $M_a,M_u\in\{0,1\}^N$ and $\odot$ denotes elementwise multiplication. One means \emph{observed}; missing normalized entries are zero-filled, with explicit masks distinguishing them from measured zeros. The implemented observations are selected directly from the released fields using the specified masks. Complete training pairs supply latent and physical targets; inference receives only $Y_M$.

\textbf{Forward recovery} observes $a$ and evaluates $u$; \textbf{inverse recovery} observes $u$ and evaluates $a$; \textbf{joint recovery} observes both and evaluates each. The model reconstructs both channels. These names specify recovery direction, not unique identifiability. Cylinder flow additionally supplies released disk values in both channels (Appendix~\ref{app:boundary}). Complete visibility of the active channel does not expose the hidden channel beyond this conditioning; observing the complete pair makes $\widetilde X$ available.

Tasks, budgets, and layouts vary while the complete-field target space stays fixed. \textbf{Uniform points, regular grids, random lines, clusters, and blocks all belong to backbone training}. Their construction and batch schedule are detailed in Appendix~\ref{app:masks}. Held-out fields and fresh masks evaluate the reported observation conditions; the distinction between backbone training and adapted-model evaluation is specified in Appendix~\ref{app:full-results}.

\subsection{Related Work}
\label{sec:related_work}

Neural operators learn function-space mappings for PDE inference~\citep{kovachki2023neural}, including FNO~\citep{li2021fno}, DeepONet~\citep{lu2021deeponet}, CNO~\citep{raonic2023cno}, and attention-based GNOT and Transolver~\citep{hao2023gnot,wu2024transolver}. Their treatment of spatial interactions and discretization is distinct from the observation process: changing numerical representations need not remove measurements, whereas partial observations change the evidence for prediction~\citep{gao2025crop}. Sparse-context training also differs from zero-filling only at evaluation. LANO uses mask-to-predict learning and latent autoregressive reconstruction with incomplete training observations~\citep{hou2026lano}; SCOPE forms partial contexts from complete paired training fields.

Generative methods combine learned field distributions with measurements during inference. DiffusionPDE conditions field generation on partial observations~\citep{huang2024diffusionpde}, while FunDPS uses guided diffusion sampling in function spaces~\citep{yao2025fundps}. PISD, PRISMA, EquiReg, and DDIS study spectral representations, physics-informed attention, equivariance regularization, and decoupled prior-and-operator guidance, respectively~\citep{gallon2026pisd,sawhney2025prisma,tolooshams2025equireg,lin2026ddis}. SCOPE instead learns a \textbf{deterministic observation-to-field predictor}, using complete pairs for latent and physical supervision. Its output is a point estimate, not a posterior sample or calibrated uncertainty estimate.

I-JEPA predicts target-block representations from an observed context~\citep{assran2023ijepa}, while data2vec predicts contextualized full-input representations from masked inputs~\citep{baevski2022data2vec}; masked autoencoding has also been studied for PDE representations~\citep{zhou2024masked}. SCOPE couples observation-conditioned prediction to shared physical reconstruction, rather than introducing complete-input targets or teacher--student asymmetry alone. Its complete-field latent targets have mask-independent coordinate support and uniform relative weights; the overall auxiliary factor $\alpha_t$ follows the gradient-balancing rule in Section~\ref{sec:scope_objective}. Its variance safeguard is inspired by VICReg's variance regularization~\citep{bardes2022vicreg}. Complementing joint-embedding and reconstruction analyses centered on prediction spaces and view generation~\citep{vanassel2025joint}, our analysis concerns conditional recovery and physical readout at fixed representations.

Martingale-consistent self-supervision studies conditional coherence under information refinement~\citep{gogl2026martingale}, while common-information learning studies information shared across views~\citep{kleinman2023gkvae}. These cross-view relations differ from selecting supervised coordinates within one context. SCOPE distinguishes common targets from identical predictions without adding a martingale or cross-view equality penalty. Section~\ref{sec:pitfalls} and Appendix~\ref{app:scope_information} formalize this distinction; Section~\ref{sec:lossaware} connects conditional prediction to physical readout.

\subsection{Physical Coverage and Latent Recoverability}
\label{sec:pitfalls}

Fix a deterministic teacher $T$ and a joint field/observation law, with $\widetilde X\in L^2(\mathbb R^{2N})$ and $Z=T(\widetilde X)\in L^2(\mathbb R^r)$. For $Y_M$ from Section~\ref{sec:partial_observation_setting}, let $\mathcal G_M=\sigma(Y_M)$ denote the completed observation sigma-field, $\mu_{\mathcal G_M}=\mathbb E[Z\mid\mathcal G_M]$, and $C_{\mathcal G_M}=\operatorname{Cov}(Z\mid\mathcal G_M)$. These are population quantities for the fixed target and observation law, not a trained student's errors.

A direction $v\in\mathbb R^r$ is \emph{exactly recoverable} if $v^\top Z=g(Y_M)$ almost surely for a measurable $g$. Its minimum squared prediction error is $v^\top\mathbb E[C_{\mathcal G_M}]v$, so exact recovery holds precisely when this vanishes. Positive uncertainty still permits informative approximation. Such a statistic may combine coordinates across tokens; spatial position alone does not certify recoverability (Appendix~\ref{app:scope_recoverability}).

Size-weighted means over nonempty groups partitioning the latent coordinates reproduce the full-target loss. Other normalized group weights yield $\ell_\Lambda=r^{-1}(\hat z-Z)^\top\Lambda_M(\hat z-Z)$, where $\Lambda_M$ is a mask-dependent nonnegative diagonal metric with $\operatorname{tr}(\Lambda_M)=r$. Appendix~\ref{app:scope_recoverability} gives the group-to-metric identity and distinguishes supervision groups from zero conditional uncertainty. SCOPE uses $\Lambda_M=I_r$: the target coordinates receive equal relative weight without requiring identical predictions or exact recovery of every feature. Merely renaming groups does not change the objective.

For nested observation sigma-fields, write $\mu_i=\mathbb E[Z\mid\mathcal G_i]$. The tower property gives $\mathbb E[\mu_\ell\mid\mathcal G_k]=\mu_k$ when $\mathcal G_k\subseteq\mathcal G_\ell$~\citep{kallenberg2021foundations}, allowing predictions to change with evidence. Complete invariance instead restricts predictions to information common to the views. Definition~\ref{def:consistency} and the corresponding excess-risk identity appear in Appendix~\ref{app:scope_information}; these cross-context relations are distinct from within-target grouping.

Squared field reconstruction has conditional target $\nu_{\mathcal G_M}=\mathbb E[\widetilde X\mid\mathcal G_M]$. Even when $D(Z)=\widetilde X$ almost surely, nonlinear decoding need not satisfy $D(\mu_{\mathcal G_M})=\nu_{\mathcal G_M}$. These squared-loss identities do not identify the optimum of the unsquared physical relative error used in training. Section~\ref{sec:lossaware} distinguishes the quadratic reference from its physical-error readout-transfer bounds.

\section{SCOPE: Field-Grounded Full-Target Prediction}
\label{sec:method}

SCOPE couples prediction of a complete-field representation with physical reconstruction from partial observations. Training comprises complete-view initialization, joint learning over mixed observation contexts, and offline decoder adaptation with a frozen backbone. Figure~\ref{fig:architecture} illustrates the training and inference pathways; numerical settings are specified in Appendix~\ref{app:implementation}.

\begin{figure}[t]
\centering
\includegraphics[width=\linewidth]{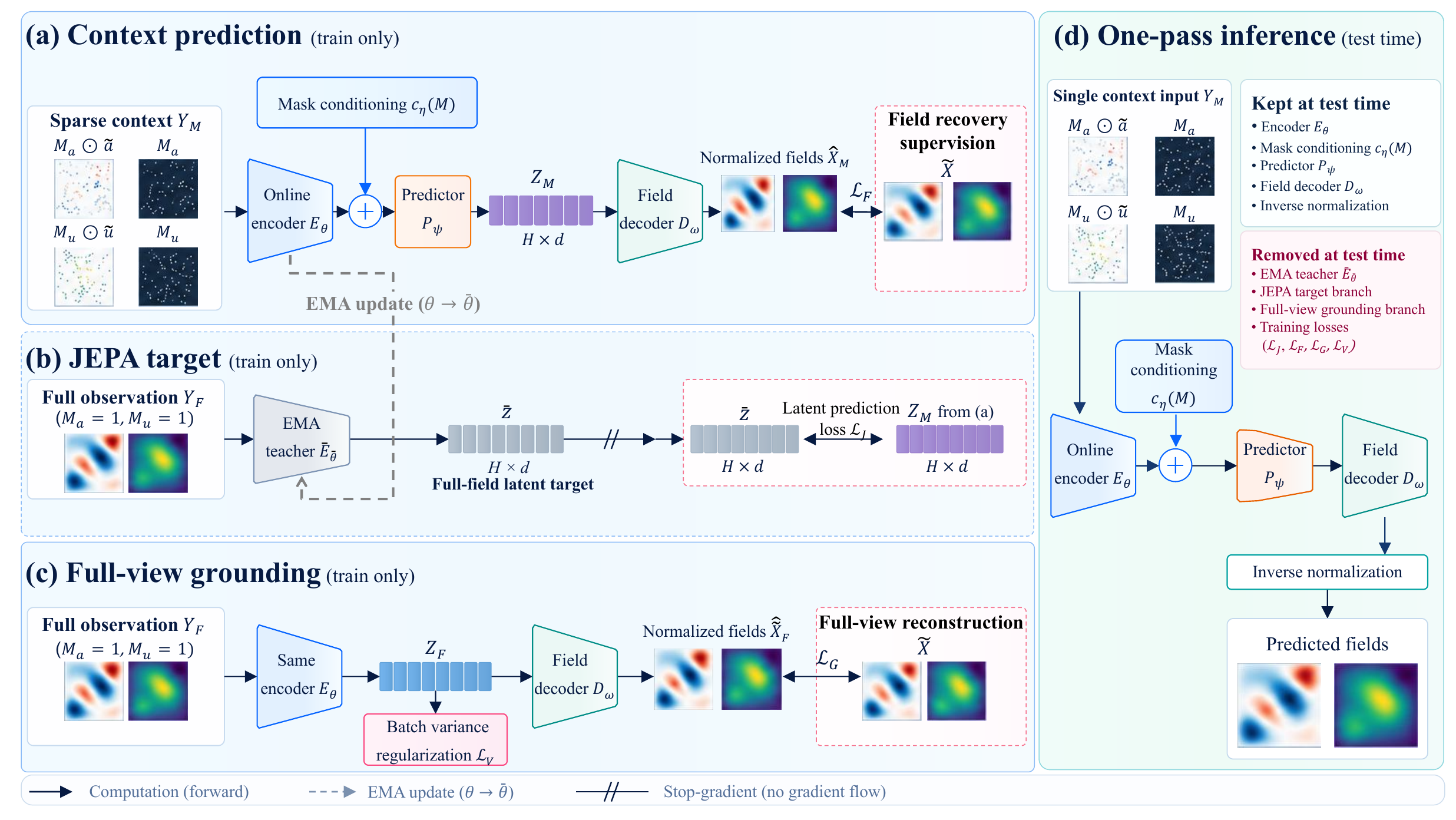}
\caption{\textbf{SCOPE training and inference.} \textbf{(a) Context prediction:} the observation-conditioned pathway reconstructs the paired fields. \textbf{(b) JEPA target:} the EMA teacher supplies $\bar z$ for comparison with $z_M$ through $\mathcal L_J$; its parameters follow the online encoder. \textbf{(c) Full-view grounding:} the shared online encoder and decoder support $\mathcal L_G$, while $z_F$ receives variance regularization. \textbf{(d) One-pass inference:} only the context pathway and inverse normalization are used. Physical field losses include inverse normalization. SP replaces $D_\omega$ with $D_{\omega'}$ before inference. Field maps are schematic.}
\label{fig:architecture}
\end{figure}

\subsection{Observation-Conditioned Full-Target Prediction}
\label{sec:fulltarget}

Let $E_\theta$, $P_\psi$, and $D_\omega$ denote the online encoder, latent predictor, and nonlinear field decoder. The encoder processes the observation context $Y_M$ from Section~\ref{sec:partial_observation_setting} as spatial tokens. A learned conditioner $c_\eta(M)$ embeds channel-wise patch visibility fractions into the same latent space. The context pathway computes $z_M=P_\psi(E_\theta(Y_M)+c_\eta(M))$ and $\widehat{\widetilde X}_M=D_\omega(z_M)$. Thus exact masks enter the encoder, while local availability summaries enter the predictor directly. The encoder and predictor use Transformer blocks; the decoder maps tokens to field tiles followed by residual convolutional refinement. The context predictor and teacher output the same latent coordinate space, allowing their representations to be compared without selecting mask-dependent target groups. Module dimensions are given in Appendix~\ref{app:architecture}.

Following context--target predictive learning~\citep{assran2023ijepa,baevski2022data2vec}, an EMA encoder receives $Y_F=(\widetilde a,\mathbf1,\widetilde u,\mathbf1)$ and produces $\bar z=\operatorname{sg}(\bar E_{\bar\theta}(Y_F))$, where $\operatorname{sg}$ stops gradients. After each online update, $\bar\theta\leftarrow\tau_t\bar\theta+(1-\tau_t)\theta$. A separate online pass computes $z_F=E_\theta(Y_F)$ and $\widehat{\widetilde X}_F=D_\omega(z_F)$. This full-view branch shares the encoder and decoder with context recovery but bypasses the predictor and conditioner.

Write $r$ for the flattened latent dimension. Empty patches remain present, and \textbf{every context predicts the entire teacher representation}. Masks condition the prediction without selecting supervised coordinates or changing their relative weights. This common target does not impose identical predictions across contexts or assume exact recovery of every feature.

Native inference returns $\widehat X_M=\mathcal N^{-1}(D_\omega(z_M))$ through the encoder, conditioner, predictor, and decoder. Within each PDE setting, the same backbone-stage checkpoint supports forward, inverse, and joint recovery. No teacher, auxiliary full-view pass, sampling, test-time optimization, or measured-value overwrite is used.

\subsection{Joint Prediction and Physical Grounding}
\label{sec:scope_objective}

Let $\langle\cdot\rangle_B$ denote a batch average and $d_\omega=\mathcal N^{-1}\circ D_\omega$ physical decoding. For nonzero target-channel norms, define $\ell_F(\widehat X,X)=\frac12\sum_{c\in\{a,u\}}\|\widehat X_c-X_c\|_2/\|X_c\|_2$. Context reconstruction, latent prediction, and full-view grounding use
\[
\begin{aligned}
\mathcal L_F&=\left\langle\ell_F(d_\omega(z_M),X)\right\rangle_B,\qquad
\mathcal L_J=\left\langle r^{-1}\|z_M-\bar z\|_2^2\right\rangle_B,\\
\mathcal L_G&=\left\langle\ell_F(d_\omega(z_F),X)\right\rangle_B.
\end{aligned}
\]
Field supervision uses unsquared physical relative $L_2$ on both channels and the complete grid, including observed locations. Appendix~\ref{app:zero} specifies the training extension for zero-energy targets. Grounding trains $z_F$ and the shared decoder for physical reconstruction; it does not directly supervise the EMA representation $\bar z$.

We also use $\mathcal L_V=r^{-1}\sum_{j=1}^r[\gamma-s_j(z_F)]_+^2$, a squared-hinge variance penalty inspired by VICReg~\citep{bardes2022vicreg}. Here $s_j$ is the batch standard deviation, $\gamma>0$ a margin, and $[q]_+=\max(q,0)$. The penalty is inactive when every coordinate meets the margin. The objective is
\[
\mathcal L=\mathcal L_F+\alpha_t\bigl(\mathcal L_J+\lambda_G\mathcal L_G+\lambda_V\mathcal L_V\bigr).
\]
The fixed coefficients weight the auxiliary terms, and the detached factor $\alpha_t\in[0,1]$ caps their summed component-gradient norms relative to the field gradient, both globally and on the encoder. Appendix~\ref{app:calibration} gives the exact rule and numerical coefficients; this is not scalar-loss balancing. Omitting an auxiliary term can also change $\alpha_t$, so comparisons concern the resulting capped training recipes rather than holding the realized auxiliary scale fixed.

Complete-view initialization trains the encoder and decoder with $\mathcal L_G+\lambda_V\mathcal L_V$. Joint training then inherits their weights, initializes the teacher from the encoder, and trains all online modules over mixed observations. The predictor and conditioner enter at random initialization. Field reconstruction updates the entire inference pathway; latent prediction updates the encoder, predictor, and conditioner; grounding updates the encoder and decoder; and variance regularization updates the encoder. Each joint-training step evaluates the context prediction, teacher target, and online full-view representation before the online parameter update and subsequent EMA update. No PDE-residual penalty is used. Optimization schedules and initialization details are specified in Appendix~\ref{app:training}.

\paragraph{Sparse-path decoder-only adaptation.}
After joint training, we freeze $(\theta,\psi,\eta)$ and replace $D_\omega$ with a newly initialized $D_{\omega'}$. SP optimizes only $\omega'$ through $\mathcal L_{\mathrm{SP}}=\langle\ell_F(d_{\omega'}(\operatorname{sg}(z_M)),X)\rangle_B$, where $d_{\omega'}=\mathcal N^{-1}\circ D_{\omega'}$. The frozen modules run without gradients. Both output fields are supervised using training pairs, without teacher evaluation or auxiliary losses. Complete held-out targets enter neither decoder fitting nor the inference pathway. Deployment uses $d_{\omega'}(z_M)$, preserving deterministic, single-pass inference. The acquisition protocol and decoder configurations are given in Appendix~\ref{app:adaptation}.

The complete-path fitting (DP) control instead fits a decoder on detached $z_F$ and tests it on $z_M$. It examines readout transfer between the representation pathways, whereas SP fits the deployment pathway directly. Section~\ref{sec:lossaware} analyzes this distinction at fixed representations.

\subsection{Common Targets and Physical Readout Adaptation}
\label{sec:lossaware}

Fix a teacher--decoder snapshot, with $Z=T(\widetilde X)$, $D=D_\omega$, and $\mathcal G=\mathcal G_M$. For fixed $\alpha>0$ and $\beta=1/(2N)$, consider $R_{\mathcal G}(f;D)=\mathbb E[\beta\|D(f)-\widetilde X\|^2+\alpha\|f-Z\|^2]$. This quadratic reference uses normalized fields, not the unsquared physical training loss; $\alpha$ is distinct from $\alpha_t$. Grounding and variance regularization shape the fixed snapshot rather than the free prediction $f$.

\paragraph{Conditional recovery.}
Assume $\widetilde X,Z\in L^2$, continuous $D$, and standard Borel observations with completed information sigma-fields. Write $\mu=\mathbb E[Z\mid\mathcal G]$, $\nu=\mathbb E[\widetilde X\mid\mathcal G]$, and $\Psi_{\mathcal G}(z)=\beta\|D(z)-\nu\|^2+\alpha\|z-\mu\|^2$. Define
\[
\begin{aligned}
V_{\mathcal G}&=\beta\mathbb E\|\widetilde X-\nu\|^2+\alpha\mathbb E\|Z-\mu\|^2,\qquad
A_{\mathcal G}(D)=\mathbb E\min_z\Psi_{\mathcal G}(z),\\
e_{\mathcal G}(f;D)&=\mathbb E[\Psi_{\mathcal G}(f)-\min_z\Psi_{\mathcal G}(z)].
\end{aligned}
\]
\begin{theorem}[Observation-conditioned risk under a common target]
\label{thm:scope_common_risk}
Every observation-measurable finite-risk predictor satisfies $R_{\mathcal G}(f;D)=V_{\mathcal G}+A_{\mathcal G}(D)+e_{\mathcal G}(f;D)$. The optimal risk is $R_{\mathcal G}^{\star}(D)=V_{\mathcal G}+A_{\mathcal G}(D)$, attained by a measurable minimizer of $\Psi_{\mathcal G}$.
\end{theorem}
The terms separate observation uncertainty, decoder incompatibility, and predictor excess risk (Appendix~\ref{app:scope_joint}). When $D(\mu)\in L^2$, a joint reference minimizer has no greater squared field risk than decoding $\mu$; the nonlinear example and strictness condition accompany Eq.~\eqref{eq:scope_decoded_mean_comparison}. Under the same target, decoder, and loss, nested information cannot increase optimal reference risk; complete-pair observation gives zero optimal risk exactly when $D(Z)=\widetilde X$ (Appendix~\ref{app:scope_information}).

\paragraph{Target weighting with shared features.}
Appendix~\ref{app:scope_recoverability} shows that positive-definite observation-measurable metrics with trace $r$ preserve the unrestricted latent-only optimum $\mu$. To examine restricted capacity, fix $\phi(Y_M)\in L^2(\mathbb R^p)$ with $\mathbb E[\phi\phi^\top]\succ0$, predictors $f_W=W\phi$ for $W\in\mathbb R^{r\times p}$, and $D_0(z)=Az+b$. For the mask-dependent nonnegative diagonal metric $\Lambda_M$ with trace $r$, define $\mathcal J_\Lambda^0(W)=\mathbb E[\beta\|D_0(W\phi)-\widetilde X\|^2+\alpha(W\phi-Z)^\top\Lambda_M(W\phi-Z)]$. Assume $\beta A^\top A+\alpha\Lambda_M\succeq cI_r$ almost surely for $c>0$. Let $W_0,W_\Lambda$ be the identity-metric and reweighted minimizers, and $\mathcal F_0(W)=\mathbb E\|D_0(W\phi)-\widetilde X\|^2$.

\begin{theorem}[Reconstruction effect of a mask-dependent target metric]
\label{thm:scope_weighting}
For $\rho=\|\widetilde X-D_0(Z)\|_{L^2}$ and $d_\Lambda=\|A(W_\Lambda-W_0)\phi\|_{L^2}$, the stated assumptions imply $\mathcal F_0(W_\Lambda)-\mathcal F_0(W_0)\geq d_\Lambda^2-2\rho d_\Lambda$. When $\rho=0$, the difference equals $d_\Lambda^2$.
\end{theorem}
This compares target metrics within a common feature class, not complete training algorithms. Under exact teacher decoding, the displacement is driven by mask-weighted approximation residuals, not irreducible uncertainty alone (Appendix~\ref{app:scope_weighting}); Appendix~\ref{app:scope_example} constructs a strict elliptic separation. Proposition~\ref{prop:scope_nonlinear_margin} bounds how nonlinear approximation and fitting errors affect this margin. Appendix~\ref{app:scope_nonlinear} also separates online grounding from teacher-decoding error and bounds loss changes caused by moving teacher targets. None of these conditions follows from grounding or EMA updates alone.

\paragraph{Physical readout transfer.}
Fix the backbone and a common joint law of fields and observations. For a nonzero physical target channel $y$, let $d$ include inverse normalization and channel selection. Define $R_F(d)=\mathbb E[\|d(z_F)-y\|_2/\|y\|_2]$, $R_M(d)=\mathbb E[\|d(z_M)-y\|_2/\|y\|_2]$, and $B_M(d)=\mathbb E[\|d(z_M)-d(z_F)\|_2/\|y\|_2]$, with finite expectations. For original and replacement decoders $d_0,d_1$, the triangle inequality yields
\begin{equation}
R_M(d_0)-R_M(d_1)\geq R_F(d_0)-R_F(d_1)-B_M(d_0)-B_M(d_1).
\label{eq:scope_minimal_readout_transfer}
\end{equation}
A complete-path gain transfers when it exceeds these discrepancies. Also, $B_M(d_0)\leq R_M(d_0)+R_F(d_0)$ relates the shared decoder's discrepancy to context reconstruction and grounding. With $\|v\|_r=\|v\|_2/\sqrt r$, let $J=\mathbb E\|z_M-\bar z\|_r^2$ and $\delta_{\mathrm{EMA}}=(\mathbb E\|\bar z-z_F\|_r^2)^{1/2}$. If $d$ is $L_d$-Lipschitz from this norm to physical field norm and $\kappa_y=(\mathbb E\|y\|_2^{-2})^{1/2}<\infty$, then
\begin{equation}
B_M(d)\leq L_d\kappa_y\bigl(\sqrt J+\delta_{\mathrm{EMA}}\bigr).
\label{eq:scope_minimal_jepa_transfer}
\end{equation}
Appendix~\ref{app:scope_readout} proves both bounds, including their application to a declared nonzero-target cohort. They give sufficient conditions for DP transfer; SP instead fits the sparse-path empirical risk directly. They do not guarantee held-out improvements or joint-training convergence.

\section{Experiments}
\label{sec:experiments}

\subsection{Experimental Setup}
\label{sec:experimental-setup}

We use the released DiffusionPDE pairs used by FunDPS~\citep{huang2024diffusionpde,yao2025fundps}: Darcy, Poisson, Helmholtz, and Navier--Stokes with and without a cylinder. Splits, normalization, and observation families are specified in Tables~\ref{tab:data}, \ref{tab:normalization}, and~\ref{tab:masks}. The primary test uniformly samples the active channel, with known-disk conditioning for cylinder flow (Appendix~\ref{app:boundary}). Forward evaluates $u$ and inverse evaluates $a$, using physical relative $L_2$ except Darcy inverse BER. Metric, zero-target, and regional conventions follow Appendices~\ref{app:metrics} and~\ref{app:zero}.

SCOPE--Full uses the native decoder; SCOPE--SP uses the primary adapted configuration in Table~\ref{tab:model}. Training budgets and architectures are listed in Tables~\ref{tab:training-campaigns}, \ref{tab:operators}, and~\ref{tab:aux-operators}. Main-table operators share our observation/task protocol; literature entries retain their source protocols (Appendix~\ref{app:baseline}). Auxiliary objectives and control configurations are specified in Appendix~\ref{app:controls}, and checkpoint, pairing, and release conventions in Appendix~\ref{app:reproducibility}.

\begin{table}[t]
\centering
\caption{\textbf{Sparse recovery across PDE settings.} Forward evaluates $u$; inverse evaluates $a$. Values are percentages: Darcy inverse BER, physical relative $L_2$ otherwise. \textbf{Bold}/\underline{underline} mark the best/second-best unrounded results within the main common-observation group. Literature rows retain their source protocols (Appendix~\ref{app:baseline}); training budgets and model sizes differ (Table~\ref{tab:training-campaigns}).}
\label{tab:main}
\begingroup
\color{black}
\small
\setlength{\tabcolsep}{3pt}
\renewcommand{\arraystretch}{1.12}
\resizebox{\linewidth}{!}{
\begin{tabular}{@{}l*{10}{r}@{}}
\toprule
\rowcolor{ScopeTint}
Method & \multicolumn{2}{c}{Darcy} & \multicolumn{2}{c}{Poisson} & \multicolumn{2}{c}{Helmholtz} & \multicolumn{2}{c}{NS} & \multicolumn{2}{c}{NS (BCs)$^\dagger$} \\
\cmidrule(lr){2-3}\cmidrule(lr){4-5}\cmidrule(lr){6-7}\cmidrule(lr){8-9}\cmidrule(l){10-11}
 & Fwd.$\downarrow$ & Inv. (BER)$\downarrow$ & Fwd.$\downarrow$ & Inv.$\downarrow$ & Fwd.$\downarrow$ & Inv.$\downarrow$ & Fwd.$\downarrow$ & Inv.$\downarrow$ & Fwd.$\downarrow$ & Inv.$\downarrow$ \\
\midrule
\scopepanel{11}{Literature-Reported: Deterministic Methods}
FNO~\citep{li2021fno} & 28.20 & 49.30 & 100.90 & 232.70 & 98.20 & 218.20 & 101.40 & 96.00 & 82.80 & 69.60 \\
PINO~\citep{li2024pino} & 35.20 & 49.20 & 107.10 & 231.90 & 106.50 & 216.90 & 101.40 & 96.00 & 81.10 & 69.50 \\
DeepONet~\citep{lu2021deeponet} & 38.30 & 41.10 & 155.50 & 105.80 & 123.10 & 132.80 & 103.20 & 97.20 & 97.70 & 91.90 \\
PINN~\citep{raissi2019pinn} & 48.80 & 59.70 & 128.10 & 130.00 & 142.30 & 160.00 & 142.70 & 146.80 & 100.10 & 105.50 \\
\addlinespace[2pt]
\scopepanel{11}{Mask-Aware Deterministic Baselines}
CNO & 12.78 & 5.71 & 35.33 & 45.09 & 33.99 & 43.43 & 26.42 & 27.55 & 33.74 & 28.25 \\
Transolver & 3.20 & 5.43 & 4.19 & 18.41 & 3.89 & 17.79 & 16.39 & 23.66 & 5.26 & 3.96 \\
FNO & 10.83 & 13.10 & 26.09 & 49.31 & 27.20 & 49.17 & 19.10 & 31.66 & 26.31 & 21.94 \\
DeepONet & 11.03 & 11.19 & 16.65 & 36.03 & 17.03 & 41.87 & 39.59 & 50.28 & 34.35 & 27.76 \\
\addlinespace[2pt]
\scopepanel{11}{Literature-Reported: Generative / Sampling Methods}
DiffusionPDE~\citep{huang2024diffusionpde} & 6.07 & 7.87 & 4.88 & 21.10 & 12.64 & 19.07 & 3.78 & 9.63 & 9.69 & 4.18 \\
FunDPS (200 steps)~\citep{yao2025fundps} & 2.88 & 6.78 & 2.04 & 24.04 & 2.20 & 20.07 & 3.99 & 9.87 & 5.91 & 4.31 \\
FunDPS (500 steps)~\citep{yao2025fundps} & 2.49 & 5.18 & 1.99 & 20.47 & 2.13 & 17.16 & 3.32 & 8.48 & 4.90 & 4.08 \\
PRISMA-50~\citep{sawhney2025prisma} & 2.90 & 6.56 & --- & --- & --- & --- & --- & --- & --- & --- \\
Equi-FunDPS~\citep{tolooshams2025equireg} & --- & --- & --- & --- & 2.12 & 15.91 & 3.06 & 7.84 & --- & --- \\
PISD~\citep{gallon2026pisd} & --- & --- & 3.08 & 13.81 & 3.47 & 12.76 & --- & --- & --- & --- \\
DDIS~\citep{lin2026ddis} & --- & --- & --- & 12.32 & --- & 12.20 & --- & 7.81 & --- & --- \\
\addlinespace[2pt]
\scopepanel{11}{Our Work}
\rowcolor{ScopeTint}
\textbf{SCOPE--Full} & \underline{2.17} & \underline{1.56} & \underline{1.68} & \underline{9.50} & \underline{1.64} & \underline{8.70} & \underline{2.58} & \underline{6.97} & \underline{2.11} & \underline{0.65} \\
\rowcolor{ScopeTint}
\textbf{SCOPE--SP-10M} & \textbf{2.17} & \textbf{1.55} & \textbf{1.67} & \textbf{9.43} & \textbf{1.63} & \textbf{8.64} & \textbf{2.57} & \textbf{6.94} & \textbf{1.97} & \textbf{0.45} \\
\bottomrule
\end{tabular}}
\endgroup
\tabnote{$^\dagger$ Cylinder flow includes known-disk values on both channels; forward error uses nonzero targets (Appendix~\ref{app:zero}). SCOPE--SP-10M is the primary adapted model (Appendix~\ref{app:architecture}). Dashes denote unreported literature entries. The ranking excludes literature rows and separate backbone ablations; shading does not denote significance.}
\end{table}

\subsection{Sparse Field Recovery}
\label{sec:main-results}

SCOPE achieves \textbf{competitive sparse forward and inverse recovery} across the evaluated elliptic and flow problems (Table~\ref{tab:main}), using deterministic, single-pass inference. The channelwise comparison also shows lower mean errors for both outputs of the evaluated SP head relative to Full (Table~\ref{tab:all-channels}): decoder adaptation improves reconstruction beyond the task's primary hidden-field endpoint. Darcy inverse BER decreases alongside continuous reconstruction error, showing that the improvement extends to thresholded permeability classification rather than only continuous-valued predictions (Tables~\ref{tab:all-channels} and~\ref{tab:app-darcy-ber}). Tables~\ref{tab:app-forward-u}--\ref{tab:app-inverse-a} report the corresponding per-record dispersion.

\begin{figure}[t]
\centering
\includegraphics[width=\linewidth]{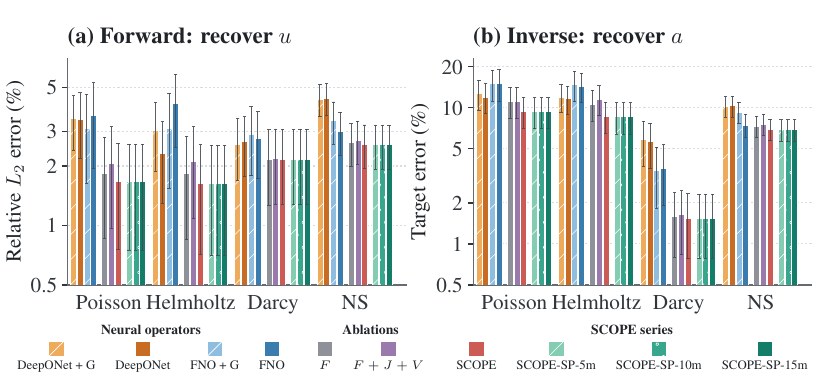}
\caption{\textbf{Grounding and readout comparisons.} Bars show mean errors with per-record SD on logarithmic axes. Forward evaluates $u$; inverse evaluates $a$, using Darcy BER. Auxiliary operators use their own protocols and metric units (Appendix~\ref{app:auxiliary-operators}); additional supervision is compared within each family. SCOPE denotes Full; all SP heads use its frozen backbone.}
\label{fig:hidden-field-current}
\label{fig:darcy-ber-current}
\end{figure}

\subsection{Grounding and Decoder Adaptation}
\label{sec:recipe-results}
\label{sec:adaptation-results}

\paragraph{Grounding.}
Full has lower mean error than FJV at every primary endpoint (Table~\ref{tab:recipes}). Both share complete-view pretraining; FJV removes main-stage grounding and recomputes auxiliary scaling. Thus the comparison supports grounding within the predictive recipe: supervising online full-view reconstruction improves the resulting sparse recovery, rather than merely fitting complete inputs (Figure~\ref{fig:hidden-field-current}; Appendices~\ref{app:training}--\ref{app:calibration}).

\paragraph{Field-only comparison.}
Full lowers all inverse mean errors relative to FO, while its forward advantage is task-dependent (Table~\ref{tab:recipes}). This supports the complete training pipeline for the evaluated inverse tasks. Because FO omits pretraining, the comparison does not isolate the latent objective.

\paragraph{Additional operator supervision.}
Auxiliary $+G$ lowers FNO's Poisson and Helmholtz forward errors but raises Darcy and flow forward errors. For DeepONet, it lowers Darcy forward error but raises forward and inverse errors on Poisson and Helmholtz (Tables~\ref{tab:app-forward-u}--\ref{tab:app-darcy-ber}). Additional reconstruction therefore does not automatically improve sparse recovery. These within-family results complement Full--FJV without treating the auxiliary objective as identical to SCOPE grounding (Appendix~\ref{app:auxiliary-operators}).

\paragraph{Fixed-backbone adaptation.}
The primary SP head lowers Full's mean error at every primary endpoint, demonstrating that physical readout can improve without changing the learned context representation (Table~\ref{tab:adaptation}). Both DP and SP improve cylinder recovery over the native decoder, while SP has lower primary mean errors than DP at every matched head size (Table~\ref{tab:capacity}). The latter comparison holds architecture and training budget fixed and supports fitting the decoder on the context representations used at deployment. Native-to-adapted comparisons additionally change capacity and optimization (Appendices~\ref{app:adaptation} and~\ref{app:runtime}).

\paragraph{Decoder capacity.}
Larger heads reduce cylinder forward and inverse errors under both DP and SP, whereas changes on other settings are small and sometimes nonmonotone (Table~\ref{tab:capacity}). Thus readout capacity is more consequential for cylinder recovery in these runs, rather than uniformly improving every PDE as the head grows.

\paragraph{Record-level changes.}
Most records improve in every reported SP--Full forward comparison, with still larger improved-record fractions for inverse recovery on Poisson, Helmholtz, and periodic NS (Tables~\ref{tab:app-paired}--\ref{tab:app-paired-inverse}). Lower aggregate means therefore accompany improvements across a majority of records. Darcy's paired continuous-error diagnostic also favors SP (Table~\ref{tab:app-paired-darcy-l2}); this is distinct from BER, and all paired statistics concern fixed models rather than training-seed variability.

\subsection{Changed Observation Conditions}
\label{sec:sensitivity}

Under Bernoulli fluid observations, the unchanged adapted cylinder heads have higher forward and inverse errors than under the primary uniform condition, while larger heads retain lower mean errors (Tables~\ref{tab:cylinder-native} and~\ref{tab:capacity}). The capacity trend thus persists beyond the adaptation acquisition law, but recovery remains sensitive to the supplied measurements. Both sensor count and sampling change, so this comparison does not isolate layout effects (Appendix~\ref{app:boundary}). Appendix~\ref{app:full-results} specifies the reported evaluation coverage.

\section{Conclusion}
\label{sec:conclusion}

SCOPE combines full-target latent prediction, shared physical reconstruction, and offline decoder adaptation for competitive sparse recovery. Complete training pairs supply latent and physical targets, while measurements and masks condition prediction. A shared nonlinear decoder connects context-predicted and online full-view representations to physical reconstruction. The reported comparisons support main-stage grounding and decoder fitting on Full. SP fits a replacement decoder on frozen, context-predicted representations, retaining deterministic, single-pass inference without test-time optimization. Matched DP--SP comparisons further support fitting the readout on the representations used at deployment. The quadratic analysis separates observation uncertainty, decoder incompatibility, and predictor excess risk at fixed teacher--decoder snapshots, while separate physical-error bounds characterize sufficient conditions for readout transfer. Together, these results distinguish representation learning from decoder fitting and clarify why common targets do not require identical predictions across observation contexts.

\paragraph{Future Work.}
Future work will extend SCOPE beyond fixed grids and assess robustness across multiple training seeds and a broader range of observation conditions for adapted models. Matched-budget, capacity-controlled comparisons can help disentangle the effects of decoder adaptation, added capacity, and additional optimization. Further theoretical work will study joint-training convergence and characterize conditions for held-out improvement beyond the current fixed-representation setting.

\label{main-text-end}

\clearpage
\ifnum\value{page}>10\relax
  \PackageError{SCOPE}{Main text exceeds nine pages}{Shorten the main text or adjust float placement; do not change the template font size or margins.}
\fi

\section*{Reproducibility Statement}
Data sources, preprocessing, observation sampling, architecture, optimization, and evaluation are specified in Appendices~\ref{app:implementation}--\ref{app:metrics}, with operator configurations in Appendix~\ref{app:operator-controls}. Theoretical assumptions and proofs are provided in Appendices~\ref{app:scope_recoverability}--\ref{app:scope_nonlinear}. We use existing DiffusionPDE paired-field datasets~\citep{huang2024diffusionpde,yao2025fundps}. Public code and data-download instructions are available at \url{https://github.com/ru1ch3n/SCOPE}; pretrained-model availability and checksums are documented in that repository.

\section*{Statement on AI Usage}
AI tools, including large language models, were used to support manuscript language refinement, literature search and discovery, and aspects of research development and execution, such as discussing experimental design, analysis procedures, and implementation approaches. The authors independently reviewed, validated, and revised all AI-assisted material and remain fully responsible for the accuracy and integrity of the methods, results, references, and conclusions reported in this paper.

\bibliography{references}
\bibliographystyle{plainnat}

\clearpage
\appendix

%

\section{Recoverability and the Full-Target Supervision Rule}
\label{app:scope_recoverability}

All expectations concern a fixed population law for the complete normalized pair and its observation. The deterministic teacher produces $Z=T(\widetilde X)$. Observations take values in standard Borel spaces, and information sigma-fields are completed. For a random vector $U$, write $\|U\|_{L^2}=(\mathbb E\|U\|^2)^{1/2}$. This notation denotes population error, not a measured validation statistic.

\subsection{Conditional Error and Recoverable Statistics}

Let $\mu=\mathbb E[Z\mid\mathcal G]$ and
$C=\mathbb E[(Z-\mu)(Z-\mu)^\top\mid\mathcal G]$.
For a fixed direction $v$ and a square-integrable observation-measurable scalar prediction $q$, conditional orthogonality gives
\begin{align}
\mathbb E|v^\top Z-q|^2
&=
\mathbb E|v^\top Z-v^\top\mu|^2
+\mathbb E|v^\top\mu-q|^2 \notag\\
&=
v^\top\mathbb E[C]v+
\mathbb E|v^\top\mu-q|^2.
\label{eq:scope_direction_risk}
\end{align}
The cross term vanishes because
$\mathbb E[Z-\mu\mid\mathcal G]=0$.
Taking $q=v^\top\mu$ proves the minimum-risk formula.
It is zero exactly when $v^\top Z=v^\top\mu$ almost surely, equivalently when $v^\top Z$ is observation-measurable. Thus the exactly recoverable fixed linear directions form $\ker\mathbb E[C]$.

For a mask-measurable orthogonal projector $P_M$,
$\mathbb E[P_MZ\mid\mathcal G]=P_M\mu$.
The vector-valued projection identity yields
\begin{equation}
\inf_q\mathbb E\|P_MZ-q(Y_M)\|^2
=
\mathbb E\operatorname{tr}(P_MCP_M).
\label{eq:scope_group_risk}
\end{equation}
Since $P_MCP_M\succeq0$, the right-hand side is zero exactly when $P_MCP_M=0$ almost surely. A rule selecting spatial token coordinates supplies $P_M$ but does not establish this condition.

For two views, a fixed direction recoverable from both lies in
\[
\ker\!\left(
\mathbb E[C_{\mathcal G_1}]
+
\mathbb E[C_{\mathcal G_2}]
\right).
\]
This follows because the sum of the two nonnegative directional residual errors vanishes exactly when each vanishes. Recoverability is stronger than agreement of conditional means: an unobserved independent centered coordinate can have equal zero predictions in both views while retaining positive uncertainty.

When masks are random, these are almost-sure statements under their joint law with the fields. A result for one fixed mask instead uses that fixed-mask law. The characterization concerns fixed linear statistics of the target, not a per-record classifier of observed and missing latent tokens.

\subsection{Partitioning versus Reweighting}

For nonempty coordinate groups $I_s(M),I_m(M)$ with projectors $P_s,P_m$ and sizes $n_s,n_m$, we have
\[
\frac1r\|f-Z\|^2
=
\frac{n_s}{r}\frac{\|P_s(f-Z)\|^2}{n_s}
+
\frac{n_m}{r}\frac{\|P_m(f-Z)\|^2}{n_m}.
\]
Thus the size-weighted grouped loss equals the full-target loss for every prediction and target.

Using coefficients $\lambda_s,\lambda_m\geq0$ with
$\lambda_s+\lambda_m=1$ instead gives
\[
\Lambda_M=
\frac{r\lambda_s}{n_s}P_s+
\frac{r\lambda_m}{n_m}P_m.
\]
The loss is $r^{-1}(f-Z)^\top\Lambda_M(f-Z)$ and
$\operatorname{tr}(\Lambda_M)=r$.
Equal group coefficients are generally not equal coordinate coefficients when group sizes differ.
At an empty-group endpoint, the rule must be specified on the remaining group without division by zero.

For any observation-measurable positive-semidefinite metric with trace $r$,
\begin{equation}
\mathbb E[(f-Z)^\top\Lambda_M(f-Z)\mid\mathcal G]
=
\operatorname{tr}(\Lambda_MC)
+
(f-\mu)^\top\Lambda_M(f-\mu).
\label{eq:scope_weighted_conditional_risk}
\end{equation}
To prove this, write $f-Z=(f-\mu)-(Z-\mu)$ and condition on $\mathcal G$.
The cross term vanishes, and the residual quadratic has conditional expectation $\operatorname{tr}(\Lambda_MC)$.
The trace bound implies $\|\Lambda_M\|_{\mathrm{op}}\leq r$, ensuring integrability for square-integrable $f$ and $Z$.

If $\Lambda_M\succ0$ almost surely, the unrestricted latent-only optimum remains $\mu$.
If it is singular, deviations in its kernel are not penalized.
These statements distinguish loss equivalence, positive reweighting, and omitted supervision; they do not identify target groups with recoverable information.

For the identity metric, the same calculation gives
\begin{equation}
\mathbb E\|f-Z\|^2
=
\mathbb E\operatorname{tr}(C)+
\mathbb E\|f-\mu\|^2.
\label{eq:scope_full_latent_risk}
\end{equation}
In particular, if $v^\top Z$ is exactly recoverable,
$\mathbb E|v^\top f-v^\top Z|^2
\leq\|v\|^2\mathbb E\|f-\mu\|^2$.
The identity criterion therefore controls all recoverable directions without requiring their explicit selection, although it does not guarantee that the trained predictor has small excess.

\section{Physical Effects of Mask-Dependent Target Weighting}
\label{app:scope_weighting}

This section fixes the same features, teacher, decoder, field loss, and observation law in both comparisons. Its rectangular class $f_W=W\phi$ contains all coefficient matrices $W\in\mathbb R^{r\times p}$. It is a reduced model of shared prediction capacity, not a claim that the complete SCOPE network is linear or trained by least squares.

\subsection{Existence and Normal Equations}

Under the assumptions of Theorem~\ref{thm:scope_weighting}, define
\[
S=\mathbb E[\phi\phi^\top],\qquad
W_Z=\mathbb E[Z\phi^\top]S^{-1},
\]
and
\[
R=\widetilde X-AZ-b,\qquad
W_R=\mathbb E[R\phi^\top]S^{-1}.
\]
Feature-projection residuals satisfy
\begin{equation}
\mathbb E[(Z-W_Z\phi)\phi^\top]=0,
\qquad
\mathbb E[(R-W_R\phi)\phi^\top]=0.
\label{eq:scope_feature_projection}
\end{equation}

Let $K_M=\beta A^\top A+\alpha\Lambda_M$ and define
\[
\mathcal T_\Lambda(U)
=
\mathbb E[K_MU\phi\phi^\top].
\]
This operator is self-adjoint under the Frobenius inner product. Moreover,
\[
\langle U,\mathcal T_\Lambda(U)\rangle_F
=
\mathbb E[(U\phi)^\top K_M(U\phi)]
\geq c\,\operatorname{tr}(USU^\top).
\]
Because $S\succ0$, the lower bound is positive for every nonzero $U$.
The objective is therefore a strictly convex coercive quadratic in $W$, with a unique minimizer.
All its coefficients are finite under the stated second-moment assumptions and the bounded metric.

For the identity metric, let $K_0=\beta A^\top A+\alpha I_r$.
Its normal equation gives
\[
K_0W_0S
=
\mathbb E[
\{\beta A^\top(\widetilde X-b)+\alpha Z\}\phi^\top].
\]
Substitution of $\widetilde X-b=AZ+R$ yields
\begin{equation}
W_0=W_Z+\beta K_0^{-1}A^\top W_R.
\label{eq:scope_identity_solution}
\end{equation}

\subsection{Proof of the Reconstruction Comparison}

\begin{proof}[Proof of Theorem~\ref{thm:scope_weighting}]
Set $\Delta=W_\Lambda-W_0$ and
\[
H=I_{2N}-\beta AK_0^{-1}A^\top
=
\left(I_{2N}+\frac{\beta}{\alpha}AA^\top\right)^{-1}.
\]
A singular-value decomposition of $A$ verifies the equality.
In particular, $H\succ0$ and $\|H\|_{\mathrm{op}}\leq1$.

Expanding the field-risk difference gives
\begin{align}
\mathcal F_0(W_\Lambda)-\mathcal F_0(W_0)
&=
\mathbb E\|A\Delta\phi\|^2 \notag\\
&\quad+
2\mathbb E\langle A\Delta\phi,
AW_0\phi+b-\widetilde X\rangle.
\label{eq:scope_field_difference_expansion}
\end{align}
By \eqref{eq:scope_identity_solution},
\[
AW_0\phi+b-\widetilde X
=
A(W_Z\phi-Z)+
\beta AK_0^{-1}A^\top W_R\phi-R.
\]
The cross term involving $W_Z\phi-Z$ vanishes by
\eqref{eq:scope_feature_projection}.
The same orthogonality permits replacing $W_R\phi$ with $R$ in the remaining cross expectation. Hence
\begin{equation}
\mathcal F_0(W_\Lambda)-\mathcal F_0(W_0)
=
d_\Lambda^2-
2\mathbb E\langle A\Delta\phi,HR\rangle.
\label{eq:scope_exact_weighting_cost}
\end{equation}
Cauchy--Schwarz and the contraction of $H$ give the stated lower bound.
When $R=0$ almost surely, the cross term is zero, proving the exact equality and its strictness condition.
\end{proof}

The exact-grounding conclusion is a projection result within the common feature class: the identity-metric predictor also minimizes field risk in that class.
Its relevance to mask-dependent supervision comes from the displacement equation below, which identifies when a reweighted fit departs from that projection.
It is not a guarantee that adding latent supervision improves a field-only predictor.

\subsection{Approximation Residuals Drive the Displacement}

Subtracting the full-target normal equation from the reweighted normal equation gives
\begin{equation}
\mathcal T_\Lambda(\Delta)
=
\alpha\mathbb E[
(\Lambda_M-I_r)(Z-W_0\phi)\phi^\top].
\label{eq:scope_weighting_displacement}
\end{equation}
Under exact grounding, $W_0=W_Z$. Since $\Lambda_M$ and $\phi$ are observation-measurable,
\[
\mathbb E[
(\Lambda_M-I_r)(Z-W_Z\phi)\phi^\top]
=
\mathbb E[
(\Lambda_M-I_r)(\mu-W_Z\phi)\phi^\top].
\]
The irreducible residual $Z-\mu$ has zero conditional mean and contributes no term.
Thus the population shift arises from the interaction of mask-dependent weights with approximation residuals of the common feature class.

If $\Lambda_M$ is constant, the right-hand side vanishes by feature orthogonality.
If $\mu=W_Z\phi$, it vanishes for every observation-measurable metric.
If the reweighting produces only a displacement annihilated by $A$, there is no field-risk difference.
These are explicit boundary cases, not exceptions hidden by the strict-improvement statement.

\subsection{A Sparse Elliptic Illustration}
\label{app:scope_example}

Let $s,t$ be independent Rademacher variables and $a=(s,t)^\top$.
Consider the two-node elliptic system
\[
Lu=a,\qquad
L=
\begin{pmatrix}
2&-1\\
-1&2
\end{pmatrix},
\qquad
K=L^{-1},
\qquad
\widetilde X=
\begin{pmatrix}
a\\Ka
\end{pmatrix}.
\]
Define
\[
O=\frac1{\sqrt2}
\begin{pmatrix}
1&1\\
1&-1
\end{pmatrix},
\qquad
Z=Oa,
\qquad
A=
\begin{pmatrix}
O^\top\\KO^\top
\end{pmatrix}.
\]
Then $D_0(z)=Az$ exactly decodes the deterministic teacher target, and
$A^\top A=\operatorname{diag}(2,10/9)$.

Independently draw $M\in\{0,1\}$ with equal probabilities.
Observe $V=s$ when $M=0$ and $V=t$ when $M=1$.
Use common features $\phi=(1,M,V)^\top$, whose second-moment matrix is
\[
S=
\begin{pmatrix}
1&1/2&0\\
1/2&1/2&0\\
0&0&1
\end{pmatrix}\succ0.
\]
The mask is available to the predictor, but the feature class omits the interaction $MV$.

For arbitrary $\alpha,\beta>0$, compare the identity metric with
\[
\Lambda_0=\operatorname{diag}(3/2,1/2),
\qquad
\Lambda_1=\operatorname{diag}(1/2,3/2).
\]
Both latent coordinates and the field loss remain supervised.
Solving the normal equations gives
\[
W_0\phi=
\begin{pmatrix}
V/\sqrt2\\0
\end{pmatrix},
\qquad
W_\Lambda\phi=
\begin{pmatrix}
V/\sqrt2\\
-\dfrac{\alpha V}
{2\sqrt2(\alpha+10\beta/9)}
\end{pmatrix}.
\]
Indeed, the weighted coefficient equation for $V$ has diagonal matrix
$\operatorname{diag}(\alpha+2\beta,\alpha+10\beta/9)$
and right-hand side
$((\alpha+2\beta)/\sqrt2,-\alpha/(2\sqrt2))^\top$.
The intercept and mask-only coefficients are zero.
Consequently,
\begin{equation}
\mathcal F_0(W_\Lambda)-\mathcal F_0(W_0)
=
\frac{45\alpha^2}{4(9\alpha+10\beta)^2}>0.
\label{eq:scope_elliptic_gap}
\end{equation}

The example also illustrates why coordinate labels do not certify recoverability.
For $M=0$,
\[
C_{\mathcal G}
=
\frac12
\begin{pmatrix}
1&-1\\
-1&1
\end{pmatrix}.
\]
Each coordinate is conditionally uncertain, yet
$(Z_1+Z_2)/\sqrt2=s$ is exactly recoverable.

This is an analytical construction, not a measured PDE benchmark result.
Adding $MV$ makes the conditional latent mean representable and removes this particular distortion.
The example demonstrates a finite-capacity mechanism, not universal inferiority of separated supervision.

\section{Nonlinear Risk and Observation Information}
\label{app:scope_joint}
\label{app:proofs}

This appendix analyzes the full-target objective with the fixed continuous nonlinear decoder, without restricting predictors to the feature class used in Appendix~\ref{app:scope_weighting}.

\subsection{Proof of the Common-Objective Decomposition}

\begin{proof}[Proof of Theorem~\ref{thm:scope_common_risk}]
For observation-measurable $f$ of finite risk, conditional orthogonality gives
\[
\mathbb E\|D(f)-\widetilde X\|^2
=
\mathbb E\|\widetilde X-\nu\|^2+
\mathbb E\|D(f)-\nu\|^2
\]
and
\[
\mathbb E\|f-Z\|^2
=
\mathbb E\|Z-\mu\|^2+
\mathbb E\|f-\mu\|^2.
\]
Combining the two identities yields
\[
R_{\mathcal G}(f;D)
=
V_{\mathcal G}+
\mathbb E\Psi_{\mathcal G}(f).
\]

To justify an attained optimum, define
\[
Q=(\sqrt\beta\,\widetilde X,\sqrt\alpha\,Z),
\quad
H_D(z)=(\sqrt\beta\,D(z),\sqrt\alpha\,z),
\quad
m=\mathbb E[Q\mid\mathcal G].
\]
Then $\Psi_{\mathcal G}(z)=\|H_D(z)-m\|^2$.
For every finite $m$, this is continuous and coercive in $z$ because $\alpha>0$.
It therefore has a nonempty compact minimizer set.
The continuous criterion has a measurable minimum and admits a measurable minimizing selection in the stated standard Borel setting.

For integrability, define $d_D(m)^2=\min_z\|H_D(z)-m\|^2$.
Comparison with $z=0$ gives
\[
d_D(m)^2\leq\|H_D(0)-m\|^2,
\]
whose expectation is finite.
For a minimizing selection $f^\star$,
\[
\|H_D(f^\star)\|^2
\leq
2d_D(m)^2+2\|m\|^2.
\]
Thus it has finite risk.
Integrating the minimum proves
$R_{\mathcal G}^\star=V_{\mathcal G}+A_{\mathcal G}$.
Subtracting this minimum gives the stated expression for predictor excess.
Existence does not imply uniqueness or convexity of the nonlinear conditional problem.
\end{proof}

The uncertainty term depends on the fixed teacher and loss coefficients.
Decoder incompatibility concerns whether the conditional latent and physical targets can be represented by the same code; it is not identical to complete-field teacher reconstruction error.

\subsection{Comparison with Latent-Mean Decoding}

Assume additionally $D(\mu)\in L^2$.
Let
\[
F_{\mathcal G}(f;D)=
\mathbb E\|D(f)-\widetilde X\|^2,
\qquad
\Gamma_{\mathcal G}(D)=
R_{\mathcal G}(\mu;D)-R_{\mathcal G}^\star(D).
\]
Using \eqref{eq:scope_full_latent_risk},
\[
R_{\mathcal G}(f;D)-R_{\mathcal G}(\mu;D)
=
\beta\{F_{\mathcal G}(f;D)-F_{\mathcal G}(\mu;D)\}
+
\alpha\mathbb E\|f-\mu\|^2.
\]
Therefore,
\begin{equation}
\beta\{F_{\mathcal G}(f;D)-F_{\mathcal G}(\mu;D)\}
=
e_{\mathcal G}(f;D)-\Gamma_{\mathcal G}(D)
-\alpha\mathbb E\|f-\mu\|^2.
\label{eq:scope_decoded_mean_comparison}
\end{equation}
At a joint optimum, the field-risk difference is nonpositive, and it is negative when $\Gamma_{\mathcal G}>0$.
For continuously differentiable $D$, a sufficient condition is
$J_D(\mu)^\top(D(\mu)-\nu)\neq0$
on a positive-probability set.
The conditional criterion then has nonzero gradient at $\mu$, so its nonnegative minimum gap is positive on that set and has positive expectation.

This compares two prediction rules for the same snapshot.
For the same admissible class,
\[
\inf_f F_{\mathcal G}(f;D)
\leq
F_{\mathcal G}(f^\star;D)
\leq
F_{\mathcal G}(\mu;D).
\]
The result therefore does not establish that adding latent supervision improves field-only learning.

\paragraph{An exactly decodable deterministic teacher.}
Let $\widetilde X$ equal $1$ or $4$ with equal probabilities, set $T(x)=\sqrt{x}$ on that support, and let $D(z)=z^2$.
Then $Z=T(\widetilde X)$ and $D(Z)=\widetilde X$ almost surely.
With uninformative observations, $\mu=3/2$, $\nu=5/2$, and $D(\mu)=9/4\neq\nu$.
Moreover,
$D'(\mu)(D(\mu)-\nu)=-3/4\neq0$.
This is a valid deterministic-teacher example of the nonlinear mismatch, not an experimental result.

\subsection{Additional Information and Equality across Contexts}
\label{app:scope_information}

\begin{definition}[Prediction across observation contexts]
\label{def:consistency}
Let $Y_i$ be observations of the same physical realization, let $\mathcal G_i=\sigma(Y_i)$ be completed observation sigma-fields, and let $f_i(Y_i)$ be square-integrable predictions in a common target space. \ding{182}~\textbf{Complete invariance} requires $f_1(Y_1)=f_2(Y_2)$ almost surely. \ding{183}~\textbf{Conditional predictive coherence} requires $\mathbb E[f_\ell(Y_\ell)\mid\mathcal G_k]=f_k(Y_k)$ almost surely whenever $\mathcal G_k\subseteq\mathcal G_\ell$.
\end{definition}

For fixed targets, decoder, coefficients, and population law, suppose $\mathcal G_k\subseteq\mathcal G_\ell$. The admissible predictors under the smaller information set form a subset of those under the larger set, so $R_\ell^\star(D)\leq R_k^\star(D)$. Writing $m_i=\mathbb E[Q\mid\mathcal G_i]$, conditional orthogonality gives
\[
V_k=V_\ell+\mathbb E\|m_\ell-m_k\|^2
\]
and hence
\[
R_k^\star(D)-R_\ell^\star(D)
=
\mathbb E\|m_\ell-m_k\|^2+A_k(D)-A_\ell(D).
\]
The incompatibility difference is not separately nonnegative. For fitted predictions, their excess-risk difference must also be included; ideal information gain need not yield a decrease in every learned error.

At complete paired-field observation, both $\widetilde X$ and $Z$ are measurable. The optimal joint risk is zero if and only if $D(Z)=\widetilde X$ almost surely. Sufficiency follows from the choice $f=Z$; necessity follows from the attained nonnegative minimum and the positivity of both loss coefficients. A larger sensor count alone does not establish information nesting: a random unordered mask may not reveal a smaller prefix unless the acquisition ordering is fixed or otherwise available.

For cross-view equality, let $\mathcal C=\mathcal G_1\cap\mathcal G_2$, $\mu_i=\mathbb E[Z\mid\mathcal G_i]$, and $\mu_{\mathcal C}=\mathbb E[Z\mid\mathcal C]$. Because the observation sigma-fields are completed, an identical prediction available from both views has a $\mathcal C$-measurable version. Its optimal value is $\mu_{\mathcal C}$. Applying conditional orthogonality for each view yields
\[
\inf_{\substack{f_i\in L^2(\mathcal G_i)\\f_1=f_2}}
\frac12\sum_{i=1}^2\mathbb E\|Z-f_i\|^2
-
\frac12\sum_{i=1}^2\mathbb E\|Z-\mu_i\|^2
=
\frac12\sum_{i=1}^2
\mathbb E\|\mu_i-\mu_{\mathcal C}\|^2.
\]
The cost is positive precisely when one of the conditional means contains target-predictive information outside the intersection. This equality restriction is different from partitioning target coordinates within one context.

\section{Grounding and Stability beyond the Affine Comparison}
\label{app:scope_nonlinear}

The comparative theorem uses a common affine decoder.
This appendix states explicit conditions under which its reconstruction margin persists for a nonlinear decoder and inexact fitting.
It does not infer these conditions merely from the presence of grounding or EMA updates.

\subsection{A Uniform Nonlinear Extension}

Let $D$ be a fixed nonlinear decoder and $D_0(z)=Az+b$ an affine reference.
Let $\mathcal W$ be a common compact convex coefficient class containing the unconstrained affine optima $W_0,W_\Lambda$.
Both methods retain the same feature map and nonlinear decoder.
Define
\[
\mathcal F_D(W)=
\mathbb E\|D(W\phi)-\widetilde X\|^2
\]
and let $\mathcal J_j^D$ replace the affine field term in $\mathcal J_j^0$ by $\beta\mathcal F_D$, for $j\in\{0,\Lambda\}$.
Assume uniformly over $W\in\mathcal W$ that
\[
\|D(W\phi)-D_0(W\phi)\|_{L^2}\leq\varepsilon_D,
\qquad
\|D_0(W\phi)-\widetilde X\|_{L^2}\leq B_F.
\]
Put $q_D=2B_F\varepsilon_D+\varepsilon_D^2$.
For each objective, assume
$\beta A^\top A+\alpha\Lambda_{j,M}\succeq c_j I_r$,
with $c_j>0$ and $\Lambda_{0,M}=I_r$.

\begin{proposition}[Persistence of a reconstruction margin]
\label{prop:scope_nonlinear_margin}
Suppose $\widehat W_j\in\mathcal W$ satisfies
\[
\mathcal J_j^D(\widehat W_j)
\leq
\inf_{W\in\mathcal W}\mathcal J_j^D(W)+\eta_j.
\]
Define
\[
\zeta_j=
q_D+
2B_F\|A\|_{\mathrm{op}}
\sqrt{\frac{2\beta q_D+\eta_j}{c_j}}.
\]
Then, using the affine comparison quantities $\rho$ and $d_\Lambda$,
\[
\mathcal F_D(\widehat W_\Lambda)
-
\mathcal F_D(\widehat W_0)
\geq
d_\Lambda^2-2\rho d_\Lambda-\zeta_\Lambda-\zeta_0.
\]
A positive right-hand side is sufficient for a strict nonlinear field-risk advantage in this common class.
\end{proposition}

\begin{proof}
Expanding squared field errors gives
\[
|\mathcal F_D(W)-\mathcal F_0(W)|
\leq q_D
\]
uniformly on $\mathcal W$.
The latent terms are unchanged, so the corresponding objectives differ by at most $\beta q_D$.
Comparing the returned predictor to the affine optimum gives
\[
\mathcal J_j^0(\widehat W_j)-\mathcal J_j^0(W_j)
\leq2\beta q_D+\eta_j.
\]
The affine objective is quadratic with vanishing first derivative at $W_j$.
Its coercivity therefore implies
\[
\mathbb E\|(\widehat W_j-W_j)\phi\|^2
\leq
\frac{2\beta q_D+\eta_j}{c_j}.
\]
The two affine field residuals have $L^2$ norms at most $B_F$.
Their squared-norm difference is consequently bounded by
\[
|\mathcal F_0(\widehat W_j)-\mathcal F_0(W_j)|
\leq
2B_F\|A\|_{\mathrm{op}}
\sqrt{\frac{2\beta q_D+\eta_j}{c_j}}.
\]
Adding the nonlinear-approximation error gives
$|\mathcal F_D(\widehat W_j)-\mathcal F_0(W_j)|\leq\zeta_j$.
Subtract the two comparisons and apply
Theorem~\ref{thm:scope_weighting}.
\end{proof}

A sufficient approximation condition follows from a $\kappa$-Lipschitz decoder Jacobian on a convex domain containing the relevant code segments.
For
$D_0(z)=D(z_\star)+J_D(z_\star)(z-z_\star)$,
Taylor's theorem gives
\[
\varepsilon_D
\leq
\frac{\kappa}{2}
\sup_{W\in\mathcal W}
\left\{\mathbb E\|W\phi-z_\star\|^4\right\}^{1/2},
\]
when the right-hand side is finite.
Smoothness alone does not make this bound small.
A pointwise approximation at a single code is insufficient for the uniform comparison.

The fitting tolerances above are population quantities.
If an empirical optimization tolerance is $\widehat\eta_j$ and a valid uniform objective-deviation bound is $u_j$, then
$\eta_j\leq\widehat\eta_j+2u_j$.
A small training objective alone supplies no such deviation bound.

\subsection{Online Grounding, Teacher Mismatch, and Moving Targets}

Let $Z_o=E_\theta(Y_F)$ and $Z=T(\widetilde X)=\bar E_{\bar\theta}(Y_F)$.
Define
\[
\rho_o=\|D(Z_o)-\widetilde X\|_{L^2},
\qquad
\Delta_{\mathrm{EMA}}=\|Z-Z_o\|_{L^2}.
\]
If $D$ is $L_D$-Lipschitz on the relevant code domain and the quantities are finite, the triangle inequality gives
\[
\|D(Z)-\widetilde X\|_{L^2}
\leq\rho_o+L_D\Delta_{\mathrm{EMA}}.
\]
For the affine-reference residual in
Theorem~\ref{thm:scope_weighting},
\[
\rho
\leq
\rho_o+L_D\Delta_{\mathrm{EMA}}
+\|D(Z)-D_0(Z)\|_{L^2}.
\]
Thus online grounding, teacher compatibility, and affine approximation contribute separately.
The method directly trains the first reconstruction pathway, not all three quantities.

The common latent rule is also distinct from a fixed numerical target over training.
For two teacher snapshots $Z_s,Z_t$ and the same prediction $f$, set
$\Delta_T=\|Z_t-Z_s\|_{L^2}$.
Expanding the squared losses gives
\[
\left|
\mathbb E\|f-Z_t\|^2-\mathbb E\|f-Z_s\|^2
\right|
\leq
2\|f-Z_s\|_{L^2}\Delta_T+\Delta_T^2.
\]
The coordinate support and weights can remain unchanged even while target values move.
This inequality does not bound the whole training trajectory, which also changes the predictor and decoder.

\subsection{Scope of the Theoretical Conclusions}

The quadratic reference in Section~\ref{sec:lossaware} separates observation uncertainty, decoder compatibility, and predictor excess risk at a fixed teacher--decoder snapshot. The target-weighting comparison additionally fixes common features and an affine decoder. Appendix~\ref{app:scope_nonlinear} extends the comparison under explicit nonlinear-approximation and fitting conditions. The reference coefficients are fixed analytical quantities, distinct from the dynamically scaled unsquared physical field objective used in training.

The physical transfer analysis below applies to a fixed backbone and compares complete-view and context-predicted representations through a decoder. Its inequalities provide sufficient conditions for DP transfer. SP instead trains the replacement decoder directly on context representations. The empirical decoder comparisons use the Full backbone; the mathematical claims retain the assumptions stated for each result.

\subsection{Physical Readout Transfer}
\label{app:scope_readout}

Fix a backbone, a field/observation law, and a nonzero target channel $y$. Let $R_F(d)$, $R_M(d)$, and $B_M(d)$ be defined as in Section~\ref{sec:lossaware}, with finite expectations. For each record, the reverse triangle inequality gives
\[
\left|\frac{\|d(z_M)-y\|_2}{\|y\|_2}-\frac{\|d(z_F)-y\|_2}{\|y\|_2}\right|
\leq\frac{\|d(z_M)-d(z_F)\|_2}{\|y\|_2}.
\]
Taking expectations yields $|R_M(d)-R_F(d)|\leq B_M(d)$. Combining $R_M(d_0)\geq R_F(d_0)-B_M(d_0)$ with $R_M(d_1)\leq R_F(d_1)+B_M(d_1)$ proves Eq.~\eqref{eq:scope_minimal_readout_transfer}. The triangle inequality also gives $B_M(d_0)\leq R_M(d_0)+R_F(d_0)$ for the shared decoder.

For the latent bound, write $\|v\|_r=\|v\|_2/\sqrt r$. If $d$ is $L_d$-Lipschitz from this norm to physical field norm and $\kappa_y=(\mathbb E\|y\|_2^{-2})^{1/2}<\infty$, then
\[
B_M(d)\leq L_d\mathbb E\frac{\|z_M-z_F\|_r}{\|y\|_2}
\leq L_d\kappa_y\left(\mathbb E\|z_M-z_F\|_r^2\right)^{1/2}.
\]
Minkowski's inequality gives
\[
\left(\mathbb E\|z_M-z_F\|_r^2\right)^{1/2}
\leq\left(\mathbb E\|z_M-\bar z\|_r^2\right)^{1/2}
+\left(\mathbb E\|\bar z-z_F\|_r^2\right)^{1/2}
=\sqrt J+\delta_{\mathrm{EMA}}.
\]
Substitution proves Eq.~\eqref{eq:scope_minimal_jepa_transfer}. The argument applies channelwise or to a declared nonzero-target evaluation cohort, subject to the stated moment and Lipschitz assumptions.

\section{Experimental Settings}
\label{app:implementation}

\subsection{Experiment groups}
\label{app:training-campaigns}

We distinguish SCOPE backbone training, the mask-aware operators in the main comparison, and the auxiliary FNO/DeepONet pairs with additional reconstruction supervision. Table~\ref{tab:training-campaigns} records their model sizes, objectives, observations, and optimization. SCOPE backbone comparisons use Full, FO, and FJV before decoder adaptation. All reported DP and SP comparisons start from the Full backbone. Operator configurations are specified in Appendix~\ref{app:controls}.

\begin{table}[t]
\centering
\caption{\textbf{Training settings by experiment group.} The main-table operators share SCOPE's observation/task protocol and physical field objective. The auxiliary operator pairs are a separate experiment.}
\label{tab:training-campaigns}
\begingroup
\color{black}
\footnotesize
\setlength{\tabcolsep}{5pt}
\renewcommand{\arraystretch}{1.16}
\begin{tabular}{@{}L{.15\linewidth}L{.265\linewidth}L{.265\linewidth}L{.24\linewidth}@{}}
\toprule
\rowcolor{ScopeTint}
Setting & SCOPE backbones & Main-table operators & Auxiliary operators \\
\midrule
Models & Full, FO, FJV & FNO, DeepONet, CNO, Transolver & FNO/DeepONet, each with and without $G$ \\
Inputs / outputs & Four masked-field/mask inputs; two fields & Six inputs including coordinates; two fields & Mask-aware inputs; one hidden-field head \\
\scopepanel{4}{Data exposure and observation contexts}
Pretraining & 10 epochs; none for FO & None & None \\
Main training & 500 epochs & 100 epochs & 100 epochs \\
Batch size & 32 & 125 & 32 \\
Task schedule & Forward/inverse/joint: $40/40/20$ & Same as SCOPE & Alternating forward/inverse: $50/50$ \\
Observation & Five balanced families; exact budgets in Table~\ref{tab:masks} & Same families, budgets, and homogeneous batches & Per-pixel Bernoulli $p=0.03$; masks refresh every five epochs \\
\scopepanel{4}{Objectives and optimization}
Field objective & Both fields; unsquared physical relative $L_2$ & Same two-field objective & Hidden field only; channel-normalized $L_1$ \\
Auxiliary terms & Recipe-dependent $J,G,V$ with gradient cap & None & Optional $0.25\,G$ using the same $L_1$ \\
Optimizer & AdamW; decay 0.05 & AdamW; decay 0.05 & AdamW; decay $10^{-5}$ \\
Learning rate & Peak $1.25\times10^{-4}$; warmup/cosine to $10^{-6}$ & Peak $1.25\times10^{-4}$; warmup/cosine & Constant $10^{-4}$ \\
Warmup & Five epochs per stage & Five epochs & None \\
Gradient clip & Norm 1 & Norm 1 & Norm 1 \\
Forward precision & BF16; losses float32 & BF16; field scores float32 & Float32; TF32 disabled \\
\bottomrule
\end{tabular}
\tabnote{The columns specify distinct model sizes and training budgets. Full-backbone decoder adaptation uses 100 epochs, batch 32, five-epoch warmup, the SCOPE optimizer, and alternating uniform/500 forward/inverse contexts. Only the new 5M/10M/15M head is trainable; its objective is field loss alone, with no teacher or auxiliary forward pass.}
\endgroup
\end{table}

\subsection{Data, splits, and physical normalization}
\label{app:data}

We use the normalized paired fields in \href{https://huggingface.co/datasets/jcy20/DiffusionPDE-normalized}{\texttt{jcy20/DiffusionPDE-normalized}}, at revision \texttt{3ededc4f2d8a1592a52a7a865cb903dab7647820}. Each float32 record has shape $2\times128\times128$. The released training and test splits are retained. Training uses shuffled training records; held-out records do not enter backbone optimization or decoder adaptation. Reported SCOPE results use terminal checkpoints without validation-based checkpoint selection.

\begin{table}[t]
\centering
\caption{\textbf{Data and training exposure.} An epoch is one training-set-equivalent number of draws. Main training and decoder adaptation use 500 and 100 epochs, respectively. Full and FJV additionally use complete-view pretraining.}
\label{tab:data}
\begingroup
\small
\renewcommand{\arraystretch}{1.15}
\begin{tabular}{@{}lrrrr@{}}
\toprule
\rowcolor{ScopeTint}
PDE & Train pairs & Test pairs & Main draws & Adaptation draws \\
\midrule
Darcy & 50,000 & 10,000 & 25,000,000 & 5,000,000 \\
Poisson & 50,000 & 1,024 & 25,000,000 & 5,000,000 \\
Helmholtz & 50,000 & 10,000 & 25,000,000 & 5,000,000 \\
NS & 50,000 & 1,000 & 25,000,000 & 5,000,000 \\
NS (BCs) & 14,000 & 1,000 & 7,000,000 & 1,400,000 \\
\bottomrule
\end{tabular}
\endgroup
\end{table}

Darcy pairs permeability $a$ with pressure $u$. Poisson and Helmholtz use source/response pairs. Obstacle-free NS uses paired vorticity snapshots, whereas cylinder NS uses earlier/later scalar speed snapshots. Physical values are recovered from the stored representation by
\begin{equation}
x_c=2\sigma_c\widetilde x_c+\mu_c.
\label{eq:normalization}
\end{equation}
SCOPE and the main-comparison operators evaluate field errors after this inverse normalization. Missing inputs are zero-filled in stored coordinates and accompanied by explicit masks. The auxiliary operators' normalization and scoring conventions are specified separately in Appendix~\ref{app:auxiliary-operators}.

\begin{table}[t]
\centering
\caption{\textbf{Normalization constants.} Physical values use the factor of two in Eq.~\eqref{eq:normalization}.}
\label{tab:normalization}
\begingroup
\small
\setlength{\tabcolsep}{6pt}
\renewcommand{\arraystretch}{1.15}
\begin{tabular}{@{}lrrrr@{}}
\toprule
\rowcolor{ScopeTint}
PDE & $\mu_a$ & $\mu_u$ & $\sigma_a$ & $\sigma_u$ \\
\midrule
Darcy & 7.5 & 0.00569201936 & 4.5 & 0.00379030361 \\
Poisson & 0 & $9.67226952\!\times\!10^{-6}$ & 0.2919494 & 0.00417478 \\
Helmholtz & 0 & $1.05050595\!\times\!10^{-5}$ & 0.2844538 & 0.00428004 \\
NS & 0 & 0 & 0.26211293 & 0.2561266 \\
NS (BCs) & 1.80715032 & 2.87310427 & 1.00997056 & 1.72194188 \\
\bottomrule
\end{tabular}
\endgroup
\end{table}

\subsection{Observation masks and batch construction}
\label{app:masks}

The input is $[M_a\odot\widetilde a,M_a,M_u\odot\widetilde u,M_u]$, where one denotes an observed location. Forward observes $a$, inverse observes $u$, and joint recovery observes both, with the cylinder convention in Appendix~\ref{app:boundary}. Joint masks have independently seeded locations for the two channels and the same nominal budget. Mask locations depend on geometry and sampling seeds rather than hidden field values.

Each batch shares one task, mask family, and budget, while sample-specific locations may differ. The schedule visits each of the five families once per five batches. Within each family, forward/inverse/joint tasks occur in proportions $40/40/20$. Uniform, regular-grid, and cluster families use the budget slots
\[
(500,500,1024,2048,4096,8192,12288,16384),
\]
so the 500-point condition has conditional weight $1/4$. Lines and blocks use 4,915 and 9,830 points, obtained by rounding $30\%$ and $60\%$ of the grid. Families have equal overall weight. Each training-set-sized sequence of draws is a fresh permutation; batch construction retains tail records and can cross an epoch boundary without mixing observation conditions.

\begin{table}[t]
\centering
\caption{\textbf{Backbone observation families.} Budgets count unique original-grid pixels before the cylinder union.}
\label{tab:masks}
\begingroup
\small
\renewcommand{\arraystretch}{1.17}
\begin{tabular}{@{}L{.17\linewidth}L{.77\linewidth}@{}}
\toprule
\rowcolor{ScopeTint}
Family & Construction \\
\midrule
Uniform & Uniform subset without replacement and with exact cardinality. \\
Regular grid & Cartesian original pixels with random phase and axis swap; no interpolation. Shapes are $20\!\times\!25$, $32\!\times\!32$, $32\!\times\!64$, $64\!\times\!64$, $64\!\times\!128$, $96\!\times\!128$, and $128\!\times\!128$. \\
Cluster & 2--4 equally weighted Gaussian components, center separation at least 0.30 domain units, and widths 0.12--0.20; weighted sampling without replacement. \\
Random lines & Whole rows or columns in random order, with at most one partial final line to meet the exact count. \\
Block & Rectangle-ordered pixels about a center in $[0.2,0.8]^2$, aspect ratio in $[2/3,3/2]$, and random tie breaking. \\
\bottomrule
\end{tabular}
\endgroup
\end{table}

Training and test mask seeds are 20260913 and 20261013. Hash-derived seeds distinguish records and conditions. Uniform/500 means 500 unique whole-grid samples before geometric conditioning, or approximately $3.05\%$ nominal coverage. Different sampled budgets need not produce nested masks.

\subsection{Known-cylinder conditioning}
\label{app:boundary}

Released per-record geometry specifies a closed disk whose values are supplied on both channels, including the otherwise hidden task channel. The active channel's sampled mask is unioned with the disk. No additional outer-boundary observations are introduced. Disk values are the released field values, rather than values inferred from the fields or overwritten with zero. Thus the nominal whole-grid sample budget and the number of visible fluid locations are distinct.

The additional \texttt{benchmark\_uniform/100} condition independently observes each fluid pixel with probability $0.01$ and supplies the same known disk. The identifier 100 denotes basis points, not a fixed number of sensors. This is a held-out acquisition condition for the adapted models, whose training uses uniform/500.

\subsection{Architecture and parameter accounting}
\label{app:architecture}

An $8\times8$ strided embedding maps the four input channels to 256 tokens. The encoder uses 12 pre-normalized Transformer blocks with width 512, 8 heads, GELU activations, MLP ratio 4, learned absolute positions, and no dropout. A linear head maps each token to width 128, yielding $16\times16\times128$ latent coordinates. Pooling is the identity because the patch and anchor grids coincide.

The conditioner is a $2\rightarrow128$ linear map applied to patchwise visibility fractions. The predictor projects width 128 to 256, applies two Transformer blocks with 8 heads and MLP ratio 4, and projects back to 128. The native decoder applies a tokenwise $128\rightarrow256\rightarrow256\rightarrow128$ MLP, forms two $8\times8$ field tiles per token, and uses residual refinement with three $3\times3$ convolutional layers of width 64. Hidden layers use GELU. No observed-field skip connection bypasses the latent representation.

Replacement decoders use four linear layers, tiled reconstruction, and five convolutional layers with residual refinement. Their MLP/convolution widths are $(1248,240)$, $(1760,352)$, and $(2168,432)$ for the 5M, 10M, and 15M heads. SCOPE--SP in the main text denotes Full with SP-10M, used consistently across PDE settings. The other sizes are decoder-capacity comparisons.

\begin{table}[t]
\centering
\caption{\textbf{Parameter accounting.} Deployment excludes the EMA teacher. Decoder adaptation updates only the replacement head.}
\label{tab:model}
\begingroup
\small
\renewcommand{\arraystretch}{1.15}
\begin{tabular}{@{}lr@{}}
\toprule
\rowcolor{ScopeTint}
Component & Parameters \\
\midrule
Online encoder & 38,157,952 \\
Predictor & 1,711,488 \\
Mask conditioner & 384 \\
Native decoder & 171,010 \\
\rowcolor{ScopeGray}
Native online model & 40,040,834 \\
EMA teacher, training-only & 38,157,952 \\
\midrule
Frozen inference backbone & 39,869,824 \\
SP-5M decoder & 5,003,170 \\
SP-5M deployment total & 44,872,994 \\
SP-10M decoder & 10,010,658 \\
\rowcolor{ScopeTint}
SP-10M deployment total & 49,880,482 \\
SP-15M decoder & 15,018,218 \\
SP-15M deployment total & 54,888,042 \\
\bottomrule
\end{tabular}
\endgroup
\end{table}

The stored frozen state can also contain the unused teacher, giving 78,027,776 non-decoder parameters. That stored-state count is distinct from the deployed model: the teacher is not evaluated during decoder adaptation or inference.

\subsection{Optimization, initialization, and backbone recipes}
\label{app:training}

SCOPE stages use batch 32 and AdamW with $(\beta_1,\beta_2)=(0.9,0.999)$, $\epsilon=10^{-8}$, peak learning rate $1.25\times10^{-4}$, and weight decay 0.05. Biases, one-dimensional parameters, and positions receive no decay. Each stage has five warmup epochs starting at $10^{-6}$, followed by cosine decay ending at $10^{-6}$. Schedules follow cumulative data draws. Combined gradients are clipped to norm 1. Forward computation uses BF16 mixed precision and losses use float32. The training seed is 20260913.

Full and FJV initialize the encoder and decoder with 10 epochs of complete-pair reconstruction plus the variance regularizer. The predictor and conditioner enter joint training at their random initialization. Main training lasts 500 epochs, resets the optimizer, and initializes the EMA teacher from the pretrained online encoder. EMA momentum increases from 0.996 to 0.9999. FO trains the same online architecture from random initialization for 500 epochs using field loss only, without pretraining or a teacher. Full and FJV differ in the presence of grounding during main training; FJV retains the same complete-view pretraining objective.

\subsection{Losses and auxiliary-gradient control}
\label{app:calibration}

The physical field objective averages the unsquared relative $L_2$ error over both channels, all spatial locations, and the batch. Grounding evaluates this same physical objective on $d_\omega(z_F)$, including inverse normalization. Latent prediction is mean squared coordinate error between $z_M$ and the stop-gradient EMA target. The variance term is
\[
\mathcal L_V=\frac1r\sum_{j=1}^{r}[0.1-s_j(z_F)]_+^2,
\]
where $s_j$ is the batch standard deviation with divisor $B$. The main-text coefficients are $\lambda_G=0.25$, $\lambda_V=0.01$, and $\gamma=0.1$. The backbone objectives are
\[
\begin{aligned}
\mathcal L_{\mathrm{Full}}&=\mathcal L_F+\alpha_t(\mathcal L_J+0.25\mathcal L_G+0.01\mathcal L_V),\\
\mathcal L_{\mathrm{FJV}}&=\mathcal L_F+\alpha_t(\mathcal L_J+0.01\mathcal L_V),\\
\mathcal L_{\mathrm{FO}}&=\mathcal L_F.
\end{aligned}
\]
For $S$ equal to all online parameters or encoder parameters alone, write
\[
f_S=\|\nabla_S\mathcal L_F\|_2,\qquad
a_S=\sum_{k\in\{J,G,V\}:w_k>0}\|w_k\nabla_S\mathcal L_k\|_2,
\qquad
\alpha_t=\min\left(1,\min_{S:a_S>0}\frac{f_S}{a_S}\right).
\]
Here $(w_J,w_G,w_V)=(1,0.25,0.01)$ for Full and $(1,0,0.01)$ for FJV. An empty inner minimum imposes no cap. A binding factor is multiplied by $1-10^{-6}$ for roundoff safety. The scalar is detached, the weighted gradients are composed, and clipping precedes AdamW. This controls the field term's share relative to the sum of component-gradient norms, both globally and on the encoder. The cap is recomputed for each recipe, so omitting grounding can also change the overall auxiliary scale.

\subsection{Full-backbone SP adaptation and DP control}
\label{app:adaptation}

Every reported decoder comparison starts from the Full epoch-500 backbone. Its native decoder is replaced with a newly initialized head. The encoder, predictor, and conditioner are frozen, and only the new decoder receives gradients. Adaptation lasts 100 epochs with a fresh optimizer and the schedule above. No teacher forward pass or latent, grounding, or variance loss is used.

SP computes detached context representations $z_M$ with the frozen inference pathway. Batches alternate forward and inverse tasks with uniform/500 observations, retaining known-cylinder conditioning. The field objective supervises both output channels. DP instead trains the head on detached complete-pair online representations $z_F$, bypassing the predictor and conditioner. Both are evaluated on context representations $z_M$. Within a decoder size, DP and SP share architecture and optimization settings but use different decoder-training representations.

The terminal epoch-100 heads are evaluated using the same held-out masks as their Full base. Adaptation masks use seed 20260913 in the training namespace; test masks use 20261013 in the test namespace. Test records do not enter adaptation. The common task, family, and nominal budget do not imply reuse of training examples or masks at evaluation.
\section{Evaluation Metrics and Comparison Protocols}
\label{app:metrics}

\subsection{Physical errors and Darcy classification}

For channel $c$, we report the mean per-record physical relative error
\begin{equation}
\operatorname{RelL2}_c=\frac{100}{|\mathcal I_c|}\sum_{i\in\mathcal I_c}\frac{\|\widehat x_{ic}-x_{ic}\|_2}{\|x_{ic}\|_2},
\qquad \mathcal I_c=\{i:\|x_{ic}\|_2>0\}.
\label{eq:rel-l2}
\end{equation}
Each norm flattens one spatial field. Forward evaluation selects $u$ and inverse evaluation selects $a$. Both output channels are also reported separately in Table~\ref{tab:all-channels}. These endpoint means are distinct from mixed-task, two-channel training losses.

Darcy permeability takes physical values 3 and 12. Predictions are thresholded at their midpoint, 7.5, to obtain binary labels. BER is the mean per-record fraction of misclassified pixels, reported as a percentage; pixel accuracy is $100-\mathrm{BER}$. Darcy inverse uses BER throughout the primary comparisons. Continuous relative $L_2$ errors on Darcy $a$ are explicitly labeled auxiliary diagnostics.

\subsection{Zero-energy cylinder targets and regional evaluation}
\label{app:zero}

The cylinder data contain 76 zero-$u$ training fields and three zero-$u$ test fields, with zero-based test indices 39, 163, and 710. No $a$ targets have zero energy. Test relative-error means use the 997 nonzero $u$ fields and all 1,000 $a$ fields. Zero-energy targets are not assigned a zero relative error or evaluated with an arbitrary denominator. The tables report this defined-target relative-error endpoint; absolute-RMSE summaries for the zero-target subset are not included.

Training retains all records. For a zero target, the channel loss is
\[
\ell_{ic}^{(0)}=\frac{\operatorname{RMSE}(\widehat x_{ic},0)}{s_{c,\mathrm{train}}},
\qquad
(s_{a,\mathrm{train}},s_{u,\mathrm{train}})=(2.070225306780491,3.3495987840120813).
\]
The zero-target flag is determined from the target data. Nonzero targets use their own physical norms. Regional diagnostics distinguish the whole grid, fluid, a fluid-side band defined by a $5\times5$ dilation of the disk intersected with fluid, and the remaining fluid interior. Regional errors use their own target norms and validity masks and are separate from the primary whole-grid endpoint.

\subsection{Numerical provenance and comparison groups}
\label{app:baseline}

The literature rows in Table~\ref{tab:main} retain the source-specific data, observation, and evaluation protocols. Sampling-step labels describe inference iterations. The literature-reported FNO, PINO, DeepONet, PINN, DiffusionPDE, and FunDPS entries are transcribed from Table~1 of \citet{yao2025fundps}; the DiffusionPDE row is their reproduced evaluation. Architecture citations identify the underlying methods rather than the source of every measurement.

PRISMA-50 uses the Darcy results in Table~10 of version~2 of \citet{sawhney2025prisma}. Equi-FunDPS uses Table~5 of arXiv version~3 of \citet{tolooshams2025equireg}, with 500 sampling steps and 100 test samples. PISD uses the 500-observation results in Tables~1--2 of \citet{gallon2026pisd}. DDIS uses the highest-runtime DDIS configuration in Table~2 of \citet{lin2026ddis}. These literature entries are excluded from the highlighted common-observation ranking and paired tests.

Our main-table mask-aware operators use the observation/task protocol and physical field objective specified in Appendix~\ref{app:implementation}, with their architecture sizes and budgets listed in Table~\ref{tab:training-campaigns}. The auxiliary FNO/DeepONet pairs use a separate Bernoulli-observation experiment and are not substituted for the main-table models. Their test populations, additional reconstruction objective, and metric units are given in Appendix~\ref{app:auxiliary-operators}.

\subsection{Decoder training and deployment cost}
\label{app:runtime}

SP-10M trains 10,010,658 decoder parameters, approximately $20.07\%$ of its 49,880,482 deployed parameters. The corresponding fractions are $11.15\%$ for SP-5M and $27.36\%$ for SP-15M. Each adaptation adds 100 epochs and uses a frozen-backbone forward pass to obtain context representations. Deployed totals include the replacement decoder and exclude the EMA teacher. Native and adapted inference each use a deterministic encoder--predictor--decoder pass. These counts report parameter and training exposure; no measured wall-clock speedup is claimed.

\section{Additional Held-Out Results}
\label{app:full-results}

\subsection{Backbone recipes before adaptation}

Table~\ref{tab:recipes} compares the native Full, FJV, and FO models. Full has lower mean error than FJV at all ten primary forward/inverse endpoints. Both use the same pretraining objective and duration; their main-stage objectives differ in the grounding term and the resulting auxiliary-gradient scale. FO is the direct field-supervised recipe described in Appendix~\ref{app:training}. Per-record statistics are provided in Appendix~\ref{app:record-statistics}.

\begin{table}[t]
\centering
\caption{\textbf{Backbone training recipes before decoder adaptation.} All rows use their own 171,010-parameter decoder at main epoch 500, the same online architecture, training data, mask schedule, and test conditions. Full is compared with field-only training (FO) and full-target prediction without main-stage grounding (FJV).}
\label{tab:recipes}
\begingroup
\color{black}
\small
\setlength{\tabcolsep}{3pt}
\renewcommand{\arraystretch}{1.12}
\resizebox{\linewidth}{!}{
\begin{tabular}{@{}l*{10}{r}@{}}
\toprule
\rowcolor{ScopeTint}
Method & \multicolumn{2}{c}{Darcy} & \multicolumn{2}{c}{Poisson} & \multicolumn{2}{c}{Helmholtz} & \multicolumn{2}{c}{NS} & \multicolumn{2}{c}{NS (BCs)$^\dagger$} \\
\cmidrule(lr){2-3}\cmidrule(lr){4-5}\cmidrule(lr){6-7}\cmidrule(lr){8-9}\cmidrule(lr){10-11}
& Fwd. & Inv. (BER) & Fwd. & Inv. & Fwd. & Inv. & Fwd. & Inv. & Fwd. & Inv. \\
\midrule
FO & 2.1625 & 1.5832 & 1.8322 & 11.2462 & 1.8306 & 10.7070 & 2.6269 & 7.3287 & 1.9829 & 0.7076 \\
FJV & 2.1788 & 1.6502 & 2.0647 & 11.2334 & 2.1258 & 11.6476 & 2.7015 & 7.5940 & 2.3897 & 0.8209 \\
\rowcolor{ScopeTint}
Full & 2.1737 & 1.5564 & 1.6804 & 9.5024 & 1.6361 & 8.7023 & 2.5813 & 6.9732 & 2.1148 & 0.6491 \\
\bottomrule
\end{tabular}}
\tabnote{Full and FJV include 10 full-view pretraining epochs; FO starts randomly without pretraining. Thus FO is a recipe control rather than an initialization-matched single-term ablation. Nominal auxiliary weights and the gradient cap are held fixed when their terms are present. Metrics and \(\dagger\) follow Table~\ref{tab:main}.}
\endgroup
\end{table}

\subsection{Decoder adaptation on the Full backbone}

Table~\ref{tab:adaptation} holds the Full encoder, predictor, and conditioner fixed and compares the native decoder with replacement heads. All replacement heads are newly initialized. SP-10M, denoted SCOPE--SP in the main text, lowers the Full model's mean error on each primary endpoint. DP and SP use complete-view and context-predicted representations, respectively, during decoder fitting; every reported sparse test uses the context pathway.

\begin{table}[t]
\centering
\caption{\textbf{Readout adaptation at a fixed Full backbone.} DP fits a new head on complete-view encoder latents; SP fits it on sparse predictor latents. Both use 100 additional epochs. Every row is tested through the sparse predictor path. Four decimal places expose small changes hidden by main-table rounding.}
\label{tab:adaptation}
\begingroup
\color{black}
\small
\setlength{\tabcolsep}{3pt}
\renewcommand{\arraystretch}{1.12}
\resizebox{\linewidth}{!}{
\begin{tabular}{@{}l*{10}{r}@{}}
\toprule
\rowcolor{ScopeTint}
Method & \multicolumn{2}{c}{Darcy} & \multicolumn{2}{c}{Poisson} & \multicolumn{2}{c}{Helmholtz} & \multicolumn{2}{c}{NS} & \multicolumn{2}{c}{NS (BCs)$^\dagger$} \\
\cmidrule(lr){2-3}\cmidrule(lr){4-5}\cmidrule(lr){6-7}\cmidrule(lr){8-9}\cmidrule(lr){10-11}
& Fwd. & Inv. (BER) & Fwd. & Inv. & Fwd. & Inv. & Fwd. & Inv. & Fwd. & Inv. \\
\midrule
Full / original head & 2.1737 & 1.5564 & 1.6804 & 9.5024 & 1.6361 & 8.7023 & 2.5813 & 6.9732 & 2.1148 & 0.6491 \\
Full / DP-5M & 2.2066 & 1.6425 & 1.6760 & 9.5222 & 1.6320 & 8.7066 & 2.5749 & 6.9797 & 2.0027 & 0.4789 \\
\rowcolor{ScopeTint}
Full / SP-5M & 2.1687 & 1.5509 & 1.6678 & 9.4277 & 1.6285 & 8.6436 & 2.5733 & 6.9402 & 1.9958 & 0.4773 \\
\rowcolor{ScopeTint}
Full / \textbf{SP-10M} & 2.1684 & 1.5503 & 1.6674 & 9.4260 & 1.6279 & 8.6425 & 2.5729 & 6.9384 & 1.9728 & 0.4524 \\
\bottomrule
\end{tabular}}
\tabnote{DP-5M and SP-5M use the same head architecture, base checkpoint, and optimizer schedule, but different decoder-training representations. Comparisons with the native decoder additionally change capacity and add 100 training epochs. SP-10M is the primary adapted model in Table~\ref{tab:main}; SP-5M and SP-15M provide capacity comparisons in Table~\ref{tab:capacity}. Metrics and $\dagger$ follow Table~\ref{tab:main}.}
\endgroup
\end{table}

\subsection{Both reconstructed channels}

Table~\ref{tab:all-channels} reports both output channels under uniform/500. Forward and inverse identify the channel evaluated in the primary comparison, while training supervises both fields. Darcy inverse BER and continuous relative $L_2$ are reported in separate columns. The SP rows here use the Full backbone and the 5M head; the full head-size comparison follows in Table~\ref{tab:capacity}.

\begingroup
\color{black}
\small
\setlength{\tabcolsep}{5pt}
\renewcommand{\arraystretch}{1.13}
\begin{longtable}{@{}llrrrrr@{}}
\caption{\textbf{Both reconstructed channels at uniform/500.} All entries are held-out percentages, not training losses. Inverse $a^\star$ uses BER for Darcy and relative $L_2$ elsewhere; all other channel errors use relative $L_2$. The final column retains Darcy inverse relative $L_2$ only as an auxiliary diagnostic. Cylinder $u$ means use 997 defined targets.}
\label{tab:all-channels}\\
\toprule
\rowcolor{ScopeTint}
PDE & Model / head & \multicolumn{2}{c}{Forward} & \multicolumn{2}{c}{Inverse} & Darcy aux. \\
\cmidrule(lr){3-4}\cmidrule(lr){5-6}
& & $a$ ($L_2$) & $u$ ($L_2$) & $a^\star$ & $u$ ($L_2$) & Inv. $a$ ($L_2$) \\
\midrule
\endfirsthead
\multicolumn{7}{l}{\tablename\ \thetable\ (continued)}\\
\toprule
\rowcolor{ScopeTint}
PDE & Model / head & \multicolumn{2}{c}{Forward} & \multicolumn{2}{c}{Inverse} & Darcy aux. \\
\cmidrule(lr){3-4}\cmidrule(lr){5-6}
& & $a$ ($L_2$) & $u$ ($L_2$) & $a^\star$ & $u$ ($L_2$) & Inv. $a$ ($L_2$) \\
\midrule
\endhead
\midrule
\multicolumn{7}{r}{Continued on next page}\\
\endfoot
\bottomrule
\endlastfoot
Darcy & FO & 15.6697 & 2.1625 & 1.5832 & 0.4097 & 10.7508 \\
Darcy & FJV & 15.7062 & 2.1788 & 1.6502 & 0.4502 & 10.9896 \\
Darcy & Full & 15.6970 & 2.1737 & 1.5564 & 0.4213 & 10.6789 \\
\rowcolor{ScopeTint}
Darcy & Full / SP-5M & 15.6806 & 2.1687 & 1.5509 & 0.4008 & 10.6478 \\
\addlinespace[4pt]
Poisson & FO & 8.5021 & 1.8322 & 11.2462 & 0.7304 & --- \\
Poisson & FJV & 9.2788 & 2.0647 & 11.2334 & 0.8510 & --- \\
Poisson & Full & 8.1587 & 1.6804 & 9.5024 & 0.4730 & --- \\
\rowcolor{ScopeTint}
Poisson & Full / SP-5M & 8.1318 & 1.6678 & 9.4277 & 0.4467 & --- \\
\addlinespace[4pt]
Helmholtz & FO & 8.5624 & 1.8306 & 10.7070 & 0.6712 & --- \\
Helmholtz & FJV & 8.9854 & 2.1258 & 11.6476 & 1.0201 & --- \\
Helmholtz & Full & 8.1235 & 1.6361 & 8.7023 & 0.3937 & --- \\
\rowcolor{ScopeTint}
Helmholtz & Full / SP-5M & 8.0989 & 1.6285 & 8.6436 & 0.3755 & --- \\
\addlinespace[4pt]
NS & FO & 7.4657 & 2.6269 & 7.3287 & 0.7303 & --- \\
NS & FJV & 7.6031 & 2.7015 & 7.5940 & 0.8488 & --- \\
NS & Full & 7.3645 & 2.5813 & 6.9732 & 0.6419 & --- \\
\rowcolor{ScopeTint}
NS & Full / SP-5M & 7.3632 & 2.5733 & 6.9402 & 0.6263 & --- \\
\addlinespace[4pt]
NS (BCs) & FO & 0.6847 & 1.9829 & 0.7076 & 1.4657 & --- \\
NS (BCs) & FJV & 0.8283 & 2.3897 & 0.8209 & 2.0319 & --- \\
NS (BCs) & Full & 0.6480 & 2.1148 & 0.6491 & 1.7350 & --- \\
\rowcolor{ScopeTint}
NS (BCs) & Full / SP-5M & 0.4749 & 1.9958 & 0.4773 & 1.6237 & --- \\
\addlinespace[4pt]
\end{longtable}
\endgroup

\subsection{Decoder capacity and training representations}

Table~\ref{tab:capacity} reports the Full-based DP and SP heads at each decoder size. All heads train for 100 epochs. On cylinder flow, the primary SP-10M configuration gives $1.9728\%$ forward and $0.4524\%$ inverse error, compared with $2.1148\%$ and $0.6491\%$ for the native decoder. A dash denotes an unreported channel-specific measurement. All available primary endpoint values are retained.

\begingroup
\color{black}
\small
\setlength{\tabcolsep}{5pt}
\renewcommand{\arraystretch}{1.13}
\begin{longtable}{@{}llrrrrr@{}}
\caption{\textbf{Decoder size and training-view sensitivity on the Full backbone.} All heads are newly initialized and trained for 100 epochs. Inverse $a^\star$ uses BER for Darcy and relative $L_2$ elsewhere; the final column retains Darcy inverse relative $L_2$ as auxiliary only. All values are percentages. A dash denotes an unreported channel-specific measurement.}
\label{tab:capacity}\\
\toprule
\rowcolor{ScopeTint}
PDE & Model / head & \multicolumn{2}{c}{Forward} & \multicolumn{2}{c}{Inverse} & Darcy aux. \\
\cmidrule(lr){3-4}\cmidrule(lr){5-6}
& & $a$ ($L_2$) & $u$ ($L_2$) & $a^\star$ & $u$ ($L_2$) & Inv. $a$ ($L_2$) \\
\midrule
\endfirsthead
\multicolumn{7}{l}{\tablename\ \thetable\ (continued)}\\
\toprule
\rowcolor{ScopeTint}
PDE & Model / head & \multicolumn{2}{c}{Forward} & \multicolumn{2}{c}{Inverse} & Darcy aux. \\
\cmidrule(lr){3-4}\cmidrule(lr){5-6}
& & $a$ ($L_2$) & $u$ ($L_2$) & $a^\star$ & $u$ ($L_2$) & Inv. $a$ ($L_2$) \\
\midrule
\endhead
\midrule
\multicolumn{7}{r}{Continued on next page}\\
\endfoot
\bottomrule
\endlastfoot
Darcy & DP-5M & 18.3474 & 2.2066 & 1.6425 & 0.4874 & 12.4596 \\
Darcy & DP-10M & 18.4497 & 2.2021 & 1.6468 & 0.4857 & 12.4984 \\
Darcy & DP-15M & 18.1494 & 2.2016 & 1.6316 & 0.4837 & 12.3332 \\
\rowcolor{ScopeTint}
Darcy & SP-5M & 15.6806 & 2.1687 & 1.5509 & 0.4008 & 10.6478 \\
\rowcolor{ScopeTint}
Darcy & SP-10M & 15.6781 & 2.1684 & 1.5503 & 0.3998 & 10.6458 \\
\rowcolor{ScopeTint}
Darcy & SP-15M & --- & 2.1683 & 1.5498 & --- & 10.6443 \\
\addlinespace[4pt]
Poisson & DP-5M & 8.1732 & 1.6760 & 9.5222 & 0.4630 & --- \\
Poisson & DP-10M & 8.1736 & 1.6753 & 9.5230 & 0.4611 & --- \\
Poisson & DP-15M & 8.1741 & 1.6756 & 9.5235 & 0.4624 & --- \\
\rowcolor{ScopeTint}
Poisson & SP-5M & 8.1318 & 1.6678 & 9.4277 & 0.4467 & --- \\
\rowcolor{ScopeTint}
Poisson & SP-10M & 8.1318 & 1.6674 & 9.4260 & 0.4453 & --- \\
\rowcolor{ScopeTint}
Poisson & SP-15M & --- & 1.6671 & 9.4249 & --- & --- \\
\addlinespace[4pt]
Helmholtz & DP-5M & 8.1268 & 1.6320 & 8.7066 & 0.3836 & --- \\
Helmholtz & DP-10M & 8.1267 & 1.6318 & 8.7067 & 0.3836 & --- \\
Helmholtz & DP-15M & 8.1268 & 1.6316 & 8.7068 & 0.3828 & --- \\
\rowcolor{ScopeTint}
Helmholtz & SP-5M & 8.0989 & 1.6285 & 8.6436 & 0.3755 & --- \\
\rowcolor{ScopeTint}
Helmholtz & SP-10M & 8.0984 & 1.6279 & 8.6425 & 0.3741 & --- \\
\rowcolor{ScopeTint}
Helmholtz & SP-15M & --- & 1.6277 & 8.6417 & --- & --- \\
\addlinespace[4pt]
NS & DP-5M & 7.3667 & 2.5749 & 6.9797 & 0.6302 & --- \\
NS & DP-10M & 7.3668 & 2.5748 & 6.9799 & 0.6293 & --- \\
NS & DP-15M & 7.3670 & 2.5751 & 6.9795 & 0.6294 & --- \\
\rowcolor{ScopeTint}
NS & SP-5M & 7.3632 & 2.5733 & 6.9402 & 0.6263 & --- \\
\rowcolor{ScopeTint}
NS & SP-10M & 7.3629 & 2.5729 & 6.9384 & 0.6256 & --- \\
\rowcolor{ScopeTint}
NS & SP-15M & --- & 2.5732 & 6.9383 & --- & --- \\
\addlinespace[4pt]
NS (BCs) & DP-5M & 0.4769 & 2.0027 & 0.4789 & 1.6318 & --- \\
NS (BCs) & DP-10M & 0.4559 & 1.9820 & 0.4576 & 1.6131 & --- \\
NS (BCs) & DP-15M & 0.4433 & 1.9784 & 0.4452 & 1.6077 & --- \\
\rowcolor{ScopeTint}
NS (BCs) & SP-5M & 0.4749 & 1.9958 & 0.4773 & 1.6237 & --- \\
\rowcolor{ScopeTint}
NS (BCs) & SP-10M & 0.4500 & 1.9728 & 0.4524 & 1.6025 & --- \\
\rowcolor{ScopeTint}
NS (BCs) & SP-15M & 0.4359 & 1.9637 & 0.4382 & 1.5927 & --- \\
\addlinespace[4pt]
\end{longtable}
\endgroup

\subsection{A changed cylinder observation condition}

Table~\ref{tab:cylinder-native} applies the same adapted heads to known-disk conditioning with Bernoulli observations of $1\%$ of the fluid region. The heads are trained under uniform/500 and evaluated here without further fitting. This acquisition law has a random fluid-sensor count and is specified in Appendix~\ref{app:boundary}.

\begin{table}[!h]
\centering
\caption{\textbf{A separate cylinder observation condition.} SP is trained only on uniform/500. These held-out rows use known-disk conditioning plus 1\% Bernoulli observations of the fluid region; ``100'' in the saved condition name denotes basis points, not 100 fixed sensors. Errors are relative $L_2$ (\%).}
\label{tab:cylinder-native}
\begingroup
\color{black}
\small
\setlength{\tabcolsep}{8pt}
\renewcommand{\arraystretch}{1.15}
\begin{tabular}{@{}lrrrr@{}}
\toprule
\rowcolor{ScopeTint}
Head & Forward $a$ & Forward $u$ & Inverse $a$ & Inverse $u$ \\
\midrule
\rowcolor{ScopeTint}
SP-5M & 1.0278 & 2.5322 & 0.7664 & 2.0122 \\
SP-10M & 1.0115 & 2.5174 & 0.7494 & 1.9952 \\
SP-15M & 1.0044 & 2.5137 & 0.7402 & 1.9895 \\
\bottomrule
\end{tabular}
\endgroup
\end{table}

\subsection{Reported evaluation conditions}

The backbone protocol supports 75 distinct task/family/budget combinations, plus two cylinder-specific acquisition conditions. The adapted-model results reported here cover forward and inverse uniform/500 recovery and the additional cylinder forward/inverse conditions. These reported conditions are distinct from the set of conditions supported by the observation generator.

\clearpage
\section{Per-Record Errors and Paired Comparisons}
\label{app:record-statistics}

Tables~\ref{tab:app-forward-u}--\ref{tab:app-darcy-ber} report the means and per-record standard deviations underlying Figure~\ref{fig:hidden-field-current}. The SCOPE groups contain native FO, FJV, and Full models and Full-based SP heads. Auxiliary operators form a separate experiment specified in Appendix~\ref{app:auxiliary-operators}. Means, units, and scored populations are identified for each group. All SP rows in this appendix use the Full backbone.

\begin{table}[!htbp]
\centering
\caption{\textbf{Forward recovery of the hidden response $u$.} Mean $\pm$ standard deviation of per-record relative $L_2$ errors (\%). SCOPE uses uniform/500 and physical units. Auxiliary operators use their own Bernoulli-$3\%$ protocol and the metric units specified below. These are not fully matched cross-family comparisons.}
\label{tab:app-forward-u}
\begingroup
\color{black}
\footnotesize
\setlength{\tabcolsep}{3pt}
\renewcommand{\arraystretch}{1.18}
\resizebox{\linewidth}{!}{
\begin{tabular}{@{}lccccc@{}}
\toprule
\rowcolor{ScopeTint}
Method & Poisson & Helmholtz & Darcy & NS (periodic) & NS (cylinder)$^\ddagger$ \\
\midrule
\scopepanel{6}{Operator references: 100 epochs; 3\% Bernoulli masks}
DeepONet + G & $3.4783\,\pm\,1.0679$ & $3.0379\,\pm\,1.1545$ & $2.5830\,\pm\,0.8954$ & $4.3536\,\pm\,0.8093$ & $13.9307\,\pm\,11.0177$ \\
DeepONet & $3.4458\,\pm\,1.2496$ & $2.3259\,\pm\,1.0340$ & $2.6641\,\pm\,0.8977$ & $4.4076\,\pm\,0.8204$ & $13.0824\,\pm\,9.6002$ \\
FNO + G & $3.1172\,\pm\,1.4950$ & $3.1115\,\pm\,1.5640$ & $2.8842\,\pm\,1.1010$ & $3.4010\,\pm\,0.8183$ & $15.1163\,\pm\,19.4803$ \\
FNO & $3.6179\,\pm\,1.6686$ & $4.1611\,\pm\,1.6845$ & $2.7522\,\pm\,1.0165$ & $3.0042\,\pm\,0.7315$ & $14.3843\,\pm\,17.5732$ \\
\scopepanel{6}{SCOPE backbone recipes: 500 main epochs; uniform/500}
FO & $1.8322\,\pm\,0.9686$ & $1.8306\,\pm\,0.9838$ & $2.1625\,\pm\,0.8977$ & $2.6269\,\pm\,0.6467$ & $1.9829\,\pm\,9.9739$ \\
FJV & $2.0647\,\pm\,1.0983$ & $2.1258\,\pm\,1.0436$ & $2.1788\,\pm\,0.8988$ & $2.7015\,\pm\,0.6595$ & $2.3897\,\pm\,10.8972$ \\
\rowcolor{ScopeGray}
Full & $1.6804\,\pm\,0.9213$ & $1.6361\,\pm\,0.9242$ & $2.1737\,\pm\,0.8969$ & $2.5813\,\pm\,0.6454$ & $2.1148\,\pm\,10.3418$ \\
\scopepanel{6}{Sparse-path adaptation: frozen Full backbone; 100 head epochs}
\rowcolor{ScopeTint}
SP-5M & $1.6678\,\pm\,0.9166$ & $1.6285\,\pm\,0.9244$ & $2.1687\,\pm\,0.8966$ & $2.5733\,\pm\,0.6473$ & $1.9958\,\pm\,10.4692$ \\
\rowcolor{ScopeTint}
SP-10M & $1.6674\,\pm\,0.9170$ & $1.6279\,\pm\,0.9244$ & $2.1684\,\pm\,0.8967$ & $2.5729\,\pm\,0.6471$ & $1.9728\,\pm\,10.4255$ \\
\rowcolor{ScopeTint}
SP-15M & $1.6671\,\pm\,0.9174$ & $1.6277\,\pm\,0.9245$ & $2.1683\,\pm\,0.8967$ & $2.5732\,\pm\,0.6473$ & $1.9637\,\pm\,10.4303$ \\
\bottomrule
\end{tabular}}
\tabnote{SCOPE uses 1,024 / 10,000 / 10,000 / 1,000 / 997 scored fields in column order and physical-unit errors. Each auxiliary operator uses 1,000 records with five mask draws, giving 5,000 record--mask pairs. SD denotes dispersion over scored rows, not independent training runs. Auxiliary Darcy errors are converted per record to physical units; periodic-NS relative errors are unchanged by zero-mean scaling. Auxiliary Poisson and Helmholtz forward errors retain stored coordinates and nonzero response shifts. $\ddagger$: cylinder auxiliary scores also retain stored coordinates and omit SCOPE's known-disk conditioning. The auxiliary norm guard excludes no records.}
\endgroup
\end{table}

\begin{table}[!htbp]
\centering
\caption{\textbf{Inverse recovery of the hidden input $a$.} Mean $\pm$ standard deviation of per-record target errors (\%): BER for Darcy, relative $L_2$ for the other PDEs. Operators use their own Bernoulli 3\% protocol; SCOPE uses exactly 500 uniform points. These are not fully matched cross-family comparisons.}
\label{tab:app-inverse-a}
\begingroup
\color{black}
\footnotesize
\setlength{\tabcolsep}{3pt}
\renewcommand{\arraystretch}{1.18}
\resizebox{\linewidth}{!}{
\begin{tabular}{@{}lccccc@{}}
\toprule
\rowcolor{ScopeTint}
Method & Poisson & Helmholtz & Darcy (BER) & NS (periodic) & NS (cylinder)$^\ddagger$ \\
\midrule
\scopepanel{6}{Operator references: 100 epochs; 3\% Bernoulli masks}
DeepONet + G & $12.7519\,\pm\,3.1407$ & $12.0210\,\pm\,2.7818$ & $5.8548\,\pm\,2.1429$ & $10.3008\,\pm\,1.7708$ & $13.4073\,\pm\,7.0904$ \\
DeepONet & $12.0413\,\pm\,3.0069$ & $11.6675\,\pm\,2.7691$ & $5.6359\,\pm\,2.0694$ & $10.3639\,\pm\,1.8348$ & $13.3365\,\pm\,7.0118$ \\
FNO + G & $15.0845\,\pm\,3.8692$ & $14.8467\,\pm\,3.7329$ & $3.4798\,\pm\,1.6474$ & $9.3371\,\pm\,1.6888$ & $11.6073\,\pm\,6.9658$ \\
FNO & $15.0947\,\pm\,3.9014$ & $14.4706\,\pm\,3.5408$ & $3.6251\,\pm\,1.7206$ & $7.5056\,\pm\,1.4283$ & $10.9079\,\pm\,6.6674$ \\
\scopepanel{6}{SCOPE backbone recipes: 500 main epochs; uniform/500}
FO & $11.2462\,\pm\,2.8743$ & $10.7070\,\pm\,2.7732$ & $1.5832\,\pm\,0.7883$ & $7.3287\,\pm\,1.3145$ & $0.7076\,\pm\,1.0292$ \\
FJV & $11.2334\,\pm\,2.9060$ & $11.6476\,\pm\,2.8693$ & $1.6502\,\pm\,0.8179$ & $7.5940\,\pm\,1.3622$ & $0.8209\,\pm\,1.0048$ \\
\rowcolor{ScopeGray}
Full & $9.5024\,\pm\,2.4330$ & $8.7023\,\pm\,2.2587$ & $1.5564\,\pm\,0.7748$ & $6.9732\,\pm\,1.2552$ & $0.6491\,\pm\,1.0402$ \\
\scopepanel{6}{Sparse-path adaptation: frozen Full backbone; 100 head epochs}
\rowcolor{ScopeTint}
SP-5M & $9.4277\,\pm\,2.4229$ & $8.6436\,\pm\,2.2466$ & $1.5509\,\pm\,0.7721$ & $6.9402\,\pm\,1.2549$ & $0.4773\,\pm\,1.0328$ \\
\rowcolor{ScopeTint}
SP-10M & $9.4260\,\pm\,2.4231$ & $8.6425\,\pm\,2.2464$ & $1.5503\,\pm\,0.7719$ & $6.9384\,\pm\,1.2550$ & $0.4524\,\pm\,1.0332$ \\
\rowcolor{ScopeTint}
SP-15M & $9.4249\,\pm\,2.4231$ & $8.6417\,\pm\,2.2462$ & $1.5498\,\pm\,0.7717$ & $6.9383\,\pm\,1.2550$ & $0.4382\,\pm\,1.0329$ \\
\bottomrule
\end{tabular}}
\tabnote{SCOPE uses 1,024 / 10,000 / 10,000 / 1,000 / 1,000 scored fields in column order. Auxiliary models use 1,000 records and five masks per record. Darcy uses measured BER; the other non-cylinder inverse relative errors have zero-mean normalizations and are scaling-invariant. $\ddagger$: cylinder auxiliary scores retain stored units and omit known-disk conditioning. SD summarizes scored records or record--mask pairs, not training-seed variation.}
\endgroup
\end{table}

\begin{table}[!htbp]
\centering
\caption{\textbf{Darcy inverse classification.} Mean $\pm$ per-record SD of binary error rate and mean pixel accuracy. SCOPE uses uniform/500 on 10,000 fields; operators use their separate 3\% protocol on 1,000 fields with five mask draws.}
\label{tab:app-darcy-ber}
\begingroup
\color{black}
\small
\renewcommand{\arraystretch}{1.16}
\begin{tabular}{@{}lrr@{}}
\toprule
\rowcolor{ScopeTint}
Method & BER (\%) $\downarrow$ & Pixel accuracy (\%) $\uparrow$ \\
\midrule
\scopepanel{3}{Operator references: own observation protocol}
DeepONet + G & $5.8548\,\pm\,2.1429$ & 94.1452 \\
DeepONet & $5.6359\,\pm\,2.0694$ & 94.3641 \\
FNO + G & $3.4798\,\pm\,1.6474$ & 96.5202 \\
FNO & $3.6251\,\pm\,1.7206$ & 96.3749 \\
\scopepanel{3}{SCOPE backbone recipes}
FO & $1.5832\,\pm\,0.7883$ & 98.4168 \\
FJV & $1.6502\,\pm\,0.8179$ & 98.3498 \\
Full & $1.5564\,\pm\,0.7748$ & 98.4436 \\
\scopepanel{3}{Sparse-path adaptation on Full}
\rowcolor{ScopeTint}
SP-5M & $1.5509\,\pm\,0.7721$ & 98.4491 \\
\rowcolor{ScopeTint}
SP-10M & $1.5503\,\pm\,0.7719$ & 98.4497 \\
\rowcolor{ScopeTint}
SP-15M & $1.5498\,\pm\,0.7717$ & 98.4502 \\
\bottomrule
\end{tabular}
\tabnote{The physical threshold is 7.5 and pixel accuracy equals $100-\mathrm{BER}$. BER is calculated from thresholded predictions, not inferred from continuous relative error. SCOPE and auxiliary operators retain their respective observation and evaluation protocols. Every SP head is attached to Full.}
\endgroup
\end{table}

\clearpage
\subsection{Paired per-record differences}

For a comparison $A$ versus $B$, let $d_i=e_i(A)-e_i(B)$ and $\Delta=\overline d$. Negative $\Delta$ favors $A$. The tables give the paired SD $s_d$, descriptive statistic $t=\Delta/(s_d/\sqrt n)$, and improved-record fraction $100\,\#\{i:d_i<0\}/n$. Pairing uses the same held-out records and observation masks for fixed trained models. Displayed means and paired statistics are independently rounded. These are record-level comparisons rather than estimates of training-seed variation. All SP comparisons use Full-based heads.

The forward table includes Poisson, Helmholtz, Darcy, and periodic NS. The primary inverse relative-error table includes Poisson, Helmholtz, and periodic NS. Cylinder paired statistics and Darcy inverse BER paired statistics are not reported. Table~\ref{tab:app-paired-darcy-l2} separately preserves the Darcy continuous-error diagnostic; it is not a BER test.

\begingroup
\color{black}
\small
\setlength{\tabcolsep}{5pt}
\renewcommand{\arraystretch}{1.10}
\begin{longtable}{@{}lrrrrr@{}}
\caption{\textbf{Forward $u$: paired differences at uniform/500.} $\Delta$ and paired SD are relative-$L_2$ percentage points; $\Delta<0$ favors the first method named in each group. Statistics are computed on paired records for the fitted models.}
\label{tab:app-paired}\\
\toprule
\rowcolor{ScopeTint}
PDE & $n$ & $\Delta$ (pp) & Paired SD & $t$ & First better (\%) \\
\midrule
\endfirsthead
\multicolumn{6}{l}{\tablename\ \thetable\ (continued)}\\
\toprule
\rowcolor{ScopeTint}
PDE & $n$ & $\Delta$ (pp) & Paired SD & $t$ & First better (\%) \\
\midrule
\endhead
\midrule
\multicolumn{6}{r}{Continued on next page}\\
\endfoot
\bottomrule
\endlastfoot
\scopepanel{6}{SP-5M vs.\ Full}
Poisson & 1,024 & $-0.0126$ & 0.0473 & $-8.5$ & 59.9 \\
Helmholtz & 10,000 & $-0.0077$ & 0.0278 & $-27.6$ & 61.9 \\
Darcy & 10,000 & $-0.0051$ & 0.0212 & $-23.9$ & 60.1 \\
NS (periodic) & 1,000 & $-0.0080$ & 0.0143 & $-17.6$ & 71.8 \\
\scopepanel{6}{SP-10M vs.\ Full}
Poisson & 1,024 & $-0.0130$ & 0.0486 & $-8.6$ & 60.0 \\
Helmholtz & 10,000 & $-0.0083$ & 0.0269 & $-30.8$ & 63.0 \\
Darcy & 10,000 & $-0.0053$ & 0.0222 & $-23.9$ & 60.3 \\
NS (periodic) & 1,000 & $-0.0083$ & 0.0146 & $-18.0$ & 72.2 \\
\scopepanel{6}{SP-15M vs.\ Full}
Poisson & 1,024 & $-0.0133$ & 0.0483 & $-8.8$ & 61.3 \\
Helmholtz & 10,000 & $-0.0084$ & 0.0266 & $-31.6$ & 63.4 \\
Darcy & 10,000 & $-0.0054$ & 0.0229 & $-23.5$ & 60.2 \\
NS (periodic) & 1,000 & $-0.0080$ & 0.0145 & $-17.5$ & 72.0 \\
\scopepanel{6}{Full vs.\ FO}
Poisson & 1,024 & $-0.1519$ & 0.3264 & $-14.9$ & 70.6 \\
Helmholtz & 10,000 & $-0.1945$ & 0.3584 & $-54.3$ & 74.6 \\
Darcy & 10,000 & $+0.0113$ & 0.1323 & $8.5$ & 46.1 \\
NS (periodic) & 1,000 & $-0.0457$ & 0.1159 & $-12.5$ & 65.1 \\
\scopepanel{6}{Full vs.\ FJV}
Poisson & 1,024 & $-0.3844$ & 0.4730 & $-26.0$ & 82.8 \\
Helmholtz & 10,000 & $-0.4896$ & 0.5110 & $-95.8$ & 88.1 \\
Darcy & 10,000 & $-0.0050$ & 0.1372 & $-3.7$ & 52.1 \\
NS (periodic) & 1,000 & $-0.1202$ & 0.1351 & $-28.2$ & 84.3 \\
\end{longtable}
\endgroup

\clearpage

\begingroup
\color{black}
\small
\setlength{\tabcolsep}{5pt}
\renewcommand{\arraystretch}{1.10}
\begin{longtable}{@{}lrrrrr@{}}
\caption{\textbf{Inverse $a$: paired differences at uniform/500.} $\Delta$ and paired SD are relative-$L_2$ percentage points; $\Delta<0$ favors the first method named in each group. Statistics are computed on paired records for the fitted models. Darcy is omitted here because its primary inverse metric is BER, for which paired summaries are unavailable.}
\label{tab:app-paired-inverse}\\
\toprule
\rowcolor{ScopeTint}
PDE & $n$ & $\Delta$ (pp) & Paired SD & $t$ & First better (\%) \\
\midrule
\endfirsthead
\multicolumn{6}{l}{\tablename\ \thetable\ (continued)}\\
\toprule
\rowcolor{ScopeTint}
PDE & $n$ & $\Delta$ (pp) & Paired SD & $t$ & First better (\%) \\
\midrule
\endhead
\midrule
\multicolumn{6}{r}{Continued on next page}\\
\endfoot
\bottomrule
\endlastfoot
\scopepanel{6}{SP-5M vs.\ Full}
Poisson & 1,024 & $-0.0747$ & 0.0306 & $-78.0$ & 99.6 \\
Helmholtz & 10,000 & $-0.0587$ & 0.0331 & $-177.5$ & 97.3 \\
NS (periodic) & 1,000 & $-0.0330$ & 0.0143 & $-73.0$ & 98.8 \\
\scopepanel{6}{SP-10M vs.\ Full}
Poisson & 1,024 & $-0.0764$ & 0.0326 & $-75.1$ & 99.3 \\
Helmholtz & 10,000 & $-0.0597$ & 0.0334 & $-179.0$ & 97.4 \\
NS (periodic) & 1,000 & $-0.0348$ & 0.0149 & $-74.0$ & 99.0 \\
\scopepanel{6}{SP-15M vs.\ Full}
Poisson & 1,024 & $-0.0775$ & 0.0327 & $-75.8$ & 99.3 \\
Helmholtz & 10,000 & $-0.0605$ & 0.0336 & $-180.1$ & 97.6 \\
NS (periodic) & 1,000 & $-0.0349$ & 0.0152 & $-72.7$ & 98.9 \\
\scopepanel{6}{Full vs.\ FO}
Poisson & 1,024 & $-1.7438$ & 0.5903 & $-94.5$ & 100.0 \\
Helmholtz & 10,000 & $-2.0047$ & 0.6197 & $-323.5$ & 100.0 \\
NS (periodic) & 1,000 & $-0.3555$ & 0.1669 & $-67.4$ & 98.9 \\
\scopepanel{6}{Full vs.\ FJV}
Poisson & 1,024 & $-1.7309$ & 0.6187 & $-89.5$ & 100.0 \\
Helmholtz & 10,000 & $-2.9453$ & 0.7408 & $-397.6$ & 100.0 \\
NS (periodic) & 1,000 & $-0.6208$ & 0.2038 & $-96.3$ & 100.0 \\
\end{longtable}
\endgroup

\begin{table}[!htbp]
\centering
\caption{\textbf{Auxiliary Darcy inverse continuous-error diagnostic (not BER).} These retained paired relative-$L_2$ statistics are not the classification endpoint used in the main comparisons. Both $\Delta$ and paired SD are relative-$L_2$ percentage points. No BER significance or improved-record rate is inferred from them.}
\label{tab:app-paired-darcy-l2}
\begingroup
\color{black}
\small
\setlength{\tabcolsep}{5pt}
\renewcommand{\arraystretch}{1.12}
\begin{tabular}{@{}lrrrrr@{}}
\toprule
\rowcolor{ScopeTint}
Comparison & $n$ & $\Delta$ (pp) & Paired SD & $t$ & First better (\%) \\
\midrule
SP-5M vs.\ Full & 10,000 & $-0.0311$ & 0.0417 & $-74.6$ & 77.8 \\
SP-10M vs.\ Full & 10,000 & $-0.0331$ & 0.0443 & $-74.7$ & 77.8 \\
SP-15M vs.\ Full & 10,000 & $-0.0346$ & 0.0461 & $-75.0$ & 78.2 \\
Full vs.\ FO & 10,000 & $-0.0719$ & 0.3327 & $-21.6$ & 58.6 \\
Full vs.\ FJV & 10,000 & $-0.3106$ & 0.3270 & $-95.0$ & 83.7 \\
\bottomrule
\end{tabular}
\endgroup
\end{table}

\clearpage
\section{Operator Configurations and Additional Supervision}
\label{app:controls}

\subsection{Main-comparison mask-aware operators}
\label{app:operator-controls}

Each operator takes six channels, $[M_a\widetilde a,M_a,M_u\widetilde u,M_u,x,y]$, and outputs dense $a,u$. Explicit masks distinguish missing values from zeros. All four families use field loss only, with no teacher, latent loss, grounding, variance term, or pretraining. They share the five-family homogeneous-batch schedule and task mixture described in Appendix~\ref{app:masks}. Training uses 100 epochs, batch 125, five-epoch warmup, AdamW peak learning rate $1.25\times10^{-4}$, and weight decay 0.05. Forward uses BF16 and physical field scoring uses float32. Model sizes and training exposure are listed separately from SCOPE in Table~\ref{tab:training-campaigns}.

\begin{table}[t]
\centering
\caption{\textbf{Main-comparison operator configurations.} These are the SCOPE-protocol models in Table~\ref{tab:main}, not the auxiliary models in Figure~\ref{fig:hidden-field-current}. Counts are real trainable degrees of freedom, including both parts of Fourier weights.}
\label{tab:operators}
\begingroup
\color{black}
\small
\renewcommand{\arraystretch}{1.16}
\begin{tabular}{@{}l r L{.59\linewidth}@{}}
\toprule
\rowcolor{ScopeTint}
Model & Parameters & Main configuration \\
\midrule
FNO & 5,029,946 & Width 40; 14 Fourier modes; 4 layers; projection width 128; padding 9. \\
DeepONet & 4,977,354 & CNN branch channels 32/64/128/192; branch hidden width 360; 256 basis functions; trunk width 128, depth 2. \\
CNO & 4,774,776 & 3 levels; 2 residual blocks; 4 bottleneck blocks; channel multiplier 52. \\
Transolver & 4,792,482 & 8 layers; width 160; 8 heads; 32 slices; MLP ratio 2. \\
\bottomrule
\end{tabular}
\endgroup
\end{table}

The main operators use the same observation families, budgets, task mixture, and physical field loss as the SCOPE backbone. The auxiliary operator pairs below instead use a distinct training objective and acquisition law.

\subsection{Auxiliary reconstruction supervision for FNO and DeepONet}
\label{app:auxiliary-operators}

\paragraph{Inputs and observations.}
The FNO/DeepONet models with and without $G$ in Figure~\ref{fig:hidden-field-current} and Tables~\ref{tab:app-forward-u}--\ref{tab:app-darcy-ber} are trained independently from the main-table operators. One head predicts the hidden field; batches alternate between observing $a$ to predict $u$ and observing $u$ to predict $a$. Zero-filled fields and explicit masks specify the input, and the target channel is fully hidden. Independent Bernoulli masks with $p=0.03$ give an expected 491.52 observations and standard deviation of approximately 21.8 on a $128\times128$ grid. Masks refresh in groups of five epochs. This differs from SCOPE's exact 500-point subsets and mixed-layout training schedule.

\paragraph{Training.}
These models train for 100 epochs from scratch, without pretraining or a teacher, using batch 32 and seed 0. AdamW uses constant learning rate $10^{-4}$, weight decay $10^{-5}$, and gradient clipping at norm 1, without warmup or cosine decay. Computation is float32 with TF32 disabled. Per-channel normalization statistics are fitted on the first 49,000 training records, or 14,000 for cylinder flow. The objective is mean $L_1$ on the single hidden field in these channel-normalized units. The $+G$ arm adds $0.25$ times the same reconstruction objective with complete observation of the input channel. This auxiliary grounding differs from SCOPE's complete-pair online representation branch and two-channel relative-error loss. Evaluation uses the terminal checkpoint without validation- or test-based checkpoint selection.

\begin{table}[t]
\centering
\caption{\textbf{Auxiliary operator parameter accounting.} These architectures are not parameter-matched and differ from Table~\ref{tab:operators}. Each complex Fourier parameter contributes two real degrees of freedom.}
\label{tab:aux-operators}
\begingroup
\color{black}
\small
\renewcommand{\arraystretch}{1.16}
\begin{tabular}{@{}lrr@{}}
\toprule
\rowcolor{ScopeTint}
Architecture & Parameter elements & Real degrees of freedom \\
\midrule
FNO / PINO & 4,803,521 & 9,522,113 \\
DeepONet & 13,239,987 & 13,239,987 \\
U-Net & 4,740,101 & 4,740,101 \\
CNO & 5,281,721 & 5,281,721 \\
\bottomrule
\end{tabular}
\tabnote{Figure~\ref{fig:hidden-field-current} uses FNO and DeepONet, each with and without grounding. U-Net, CNO, and PINO belong to the wider auxiliary study; their configurations do not supply additional results in the main comparison.}
\endgroup
\end{table}

\paragraph{Evaluation and units.}
The evaluator uses the first 1,000 records of the corresponding released held-out data, five independent mask draws per record, batch 16, and mask-seed base 12345. Within each repeat, forward and inverse share the same pixel pattern on opposite channels. The resulting 5,000 record--mask pairs are not 5,000 independent physical samples. SCOPE instead uses the test populations in Table~\ref{tab:data} with one fixed mask realization. Error bars describe dispersion over each experiment's scored rows, not comparable confidence intervals or training-seed variation.

The auxiliary relative errors are initially computed in the released stored representation. Relative $L_2$ is invariant to common scaling but not translation, so a summary mean cannot generally be converted by one global factor. Darcy errors are converted to physical units per record, and BER is measured separately using the corresponding threshold. For Poisson and Helmholtz, zero source means make inverse relative errors invariant to the stored-to-physical scaling, but nonzero response means prevent exact equality between stored-coordinate and physical forward relative errors. Periodic NS has zero channel means, so both relative errors are scaling-invariant. Cylinder auxiliary errors remain in stored coordinates and omit the known-disk conditioning used by SCOPE; they are retained only as labeled references. The auxiliary norm guard flags target norms at most $10^{-8}$ and excludes no stored-coordinate evaluation records.

\paragraph{Other auxiliary arms.}
The wider study includes CNO/U-Net on Poisson and periodic NS, and PINO on Poisson and Helmholtz, each PINO task with and without $G$ and physics weight 1.0 or 0.1. PINO uses a five-point Laplacian with $h=1/(S-1)$ and residual $\Delta u-a$ for Poisson or $\Delta u+u-a$ for Helmholtz, with dataset-refitted normalization constants. The complementary field not predicted by the single-output head is supplied from the complete training target pair when forming the residual. This is an additional training-information assumption, not access to complete test fields. No Darcy or NS PINO arm is included in this auxiliary study.

\paragraph{Interpretation of additional reconstruction.}
The $+G$ pairs evaluate full-input reconstruction supervision within each auxiliary operator family. Their mean errors improve on some endpoints and worsen on others, as shown in Tables~\ref{tab:app-forward-u}--\ref{tab:app-darcy-ber}. The Full--FJV comparison separately evaluates grounding within SCOPE's joint predictive recipe. These comparisons concern their respective training configurations; the operator $+G$ objective is not identical to SCOPE's latent-space grounding branch.
\clearpage

\section{Reproducibility and Reporting Conventions}
\label{app:reproducibility}

\paragraph{Experiment identity.}
Each result is identified by its PDE, backbone recipe, terminal checkpoint, decoder, observation condition, test split, and metric. Full, FO, and FJV use main-stage epoch-500 checkpoints with their native decoders. Every reported DP or SP experiment uses a fixed Full epoch-500 backbone and its own decoder-stage epoch-100 checkpoint. Complete-view reconstruction, sparse held-out recovery, and training losses are separate quantities.

\paragraph{Comparison groups.}
The main-table mask-aware operators and auxiliary FNO/DeepONet pairs are separately trained models. Table~\ref{tab:training-campaigns} and Appendix~\ref{app:controls} specify their inputs, objectives, sampling, architectures, and optimization. Literature values retain the source protocols identified in Appendix~\ref{app:baseline}.

\paragraph{Statistical conventions.}
Means and standard deviations summarize the scored records under the stated condition. Paired statistics use matched records for fixed models. SCOPE uses its reported training seed and fixed test masks. Rankings use unrounded values; means and paired differences may round independently. Darcy classification and continuous-error statistics remain distinct. The zero-target convention is specified in Appendix~\ref{app:zero}.

\paragraph{Availability.}
The study uses existing DiffusionPDE paired-field datasets~\citep{huang2024diffusionpde,yao2025fundps}. Public implementation, environment setup, and dataset-download instructions are available at \url{https://github.com/ru1ch3n/SCOPE}. The repository documents pretrained-model availability, checksums, and inference-loading conventions. Dataset files remain with their original provider under the upstream license.

\end{document}